\documentclass[11pt,pmlr]{jmlr}

\usepackage{times}
\usepackage{mathtools,booktabs,array}
\usepackage{microtype}
\jmlrproceedings{Preprint}{Preprint}
\jmlrvolume{}
\jmlrpages{}
\jmlryear{2026}
\jmlrworkshop{}
\editor{}
\jmlrauthors{Guangjian Zhang, Zhengyi Yang, and Xiaoyang Wang}
\title[Near-Linear Agnostic Compression for Squared Loss]{An Agnostic Sample Compression Scheme for Squared Loss\\of Near-Linear Size in the Fat-Shattering Dimension}
\author{%
\Name{Guangjian Zhang} \Email{guangjian.zhang@student.unsw.edu.au}\\
\addr UNSW Sydney
\AND
\Name{Zhengyi Yang} \Email{zhengyi.yang@sydney.edu.au}\\
\addr The University of Sydney
\AND
\Name{Xiaoyang Wang} \Email{xiaoyang.wang1@unsw.edu.au}\\
\addr UNSW Sydney
}
\newcommand{\F}{\mathcal F}
\newcommand{\X}{\mathcal X}
\newcommand{\GS}{\mathcal G_S}
\newcommand{\fat}{\operatorname{fat}}
\newcommand{\fatd}{\operatorname{fat}^{*}}
\newcommand{\opt}{\operatorname{opt}}
\newcommand{\Sel}{\operatorname{Select}}

\newcommand{\ind}[1]{\mathbf 1\!\left\{#1\right\}}
\newcommand{\eps}{\varepsilon}

\allowdisplaybreaks[2]
\hypersetup{pdftitle={Near-Linear Agnostic Compression for Squared Loss},pdfauthor={Guangjian Zhang, Zhengyi Yang, Xiaoyang Wang},bookmarksnumbered=true}
\usepackage{etoolbox}
\makeatletter
\apptocmd{\@begintheorem}{\phantomsection}{}{
  \PackageError{theorem-links}{Cannot attach a theorem destination}{}
}
\apptocmd{\@opargbegintheorem}{\phantomsection}{}{
  \PackageError{theorem-links}{Cannot attach a named theorem destination}{}
}
\makeatother

\newcommand{\stagelead}[1]{%
  \par\addvspace{1.5ex plus 0.5ex minus .2ex}%
  \noindent\textit{#1}\quad\ignorespaces}
\usepackage[all]{hypcap}

\begin{document}
\maketitle
\begin{abstract}%
We study the squared-loss sample compression question posed by Attias, Hanneke, Kontorovich, and Sadigurschi in their ICML 2024 paper. For every nonempty class $\F\subseteq[0,1]^{\X}$ and accuracy $0<\alpha\le1$ with $d=\fat(\F,c'\alpha)<\infty$, where $c'>0$ is universal, we prove the existence of a deterministic agnostic scheme of total size $O(d\log^3(2/\alpha))$. Its code contains original labeled observations and auxiliary bits, and its reconstruction has empirical loss at most $\alpha$ above the infimum over $\F$ on every finite sample. This establishes the polylogarithmic formulation in Section 5 of arXiv:1810.01864v2. The construction approximates a near-optimal internal target by a fixed-step boosting ensemble and takes its equal-weight median. Average loss permits an accuracy-dependent number of rounds, so the entire ensemble can be retained without dual-dimension sparsification. Quantized target constraints and round-membership bits determine the reconstructed function exactly. We also prove that, when retained observations must come from the sample, a fixed class of fat-shattering dimension one requires total size $\Omega(\log(1/\alpha))$. With observations and auxiliary bits charged together, this rules out the proceedings version's log-free formulation under its original-example restriction. The framework extends to average power loss at scale $\alpha/p$.
\end{abstract}
\section{Introduction}
Sample compression studies whether a predictor can be recovered from a small part of its training sample, together with finite auxiliary information \citep{littlestone1986relating}. We study the squared-loss compression question posed by \citet[Section~5]{attias2024agnostic,attias2024agnostic_arxiv} in their ICML 2024 paper and its preprint version: how closely can compression size follow the primal fat-shattering dimension at the requested accuracy scale? In agnostic regression, the observed labels need not agree with any function in the class. An approximate scheme must preserve performance relative to the infimum of empirical loss over the class. We consider bounded squared loss and codes consisting of original labeled observations and auxiliary bits, with total size bounded independently of the sample length.

For real-valued classes, the relevant complexity depends on the prediction scale. Fat-shattering dimension measures the number of input locations on which separated prediction choices can be realized independently \citep{alon1997scale,bartlett1998prediction}. Combining learner-to-compression reductions with scale-sensitive learning bounds gives primal-dimension bounds with polynomial dependence on inverse accuracy \citep{david2016supervised,alon1997scale}. Boosting-based constructions obtain more favorable accuracy dependence for realizable pointwise approximation \citep{hanneke2019sample} and agnostic empirical loss \citep{attias2024agnostic}. Their bounds involve the dual fat-shattering dimension through ensemble sparsification. This dimension measures evaluation complexity on the hypothesis class and can be exponentially larger than the primal dimension. Section~\ref{sec:related} compares the guarantees and the resulting bounds in detail.

We construct a deterministic scheme for empirical squared loss with linear primal-dimension dependence and polylogarithmic dependence on inverse accuracy. We also show that an additional accuracy cost is unavoidable, even for a fixed class of fat-shattering dimension one. The results below give the precise bounds and their relation to the two formulations of the question.

\subsection{Main Results}\label{sec:main-results}
\stagelead{A primal-dimension compression bound.}
For a nonempty class $\F\subseteq[0,1]^{\X}$ and a sequence $S=((x_i,y_i))_{i=1}^m$ of positive length, write
\[
L_2(f,S)=\frac1m\sum_{i=1}^m(f(x_i)-y_i)^2,
\qquad \opt(S)=\inf_{f\in\F}L_2(f,S).
\]
Our first result bounds the complete description. Section~\ref{sec:setting} specifies the compression model and fat-shattering convention.
\begin{theorem}[Agnostic squared-loss compression]\label{thm:main}
There are universal constants $c',C>0$ such that the following holds. Let $\varnothing\ne\F\subseteq[0,1]^{\X}$ and $0<\alpha\le1$, and suppose $d=\fat(\F,c'\alpha)<\infty$. There are deterministic maps $\kappa_\alpha$ and $\rho_\alpha$ such that, for every $m\ge1$ and every $S\in(\X\times[0,1])^m$, the code $(S',B)=\kappa_\alpha(S)$ consists of original labeled observations selected from $S$ and a finite bit string. Its reconstruction $\hat f=\rho_\alpha(S',B):\X\to[0,1]$ satisfies
\[
L_2(\hat f,S)\le\opt(S)+\alpha,
\qquad |S'|+|B|\le C d\ln^3(2/\alpha).
\]
When $d=0$, the code is empty. The reconstruction may lie outside $\F$.
\end{theorem}

The bound is linear in the primal dimension at a scale proportional to the requested accuracy, and its remaining accuracy dependence is polylogarithmic. It holds for arbitrary observed labels and all positive sample lengths. More specifically, the construction retains
\[
|S'|=O(d\ln^2(2/\alpha)),\qquad
|B|=O(d\ln^3(2/\alpha))
\]
original observations and auxiliary bits, respectively. Proposition~\ref{prop:size} accounts for the complete description.

\stagelead{A necessary dependence on accuracy.}
The accuracy dependence cannot be removed entirely. For the class of all $[0,1]$-valued functions on a singleton domain, the fat-shattering dimension is one at every scale $0<t\le1/2$. Theorem~\ref{thm:accuracy-lower} shows that any deterministic $\alpha$-approximate scheme for this fixed class has worst-case total size $\Omega(\log(1/\alpha))$ as $\alpha$ tends to zero. Corollary~\ref{cor:no-log} consequently rules out a universal bound by the primal dimension alone.

These results distinguish the two formulations in \citet[Section~5]{attias2024agnostic,attias2024agnostic_arxiv}. Theorem~\ref{thm:main} establishes the target $\fat(\F,c\alpha)\operatorname{Polylog}(c/\alpha)$ stated in the preprint arXiv:1810.01864v2 \citep{attias2024agnostic_arxiv}. Corollary~\ref{cor:no-log} rules out the literal log-free target $\fat(\F,c\alpha)$ in the proceedings version \citep{attias2024agnostic}, under the same accounting of retained original observations and auxiliary bits. Section~\ref{sec:accuracy-lower} states this comparison in the model of Definition~\ref{def:compression}, and Appendix~\ref{app:accuracy-lower} proves the lower bound.

\stagelead{Extension to average power loss.}
For each finite $p\ge1$, define $L_p(f,S)=m^{-1}\sum_i|f(x_i)-y_i|^p$, without taking a root. Under the separate assumption $\fat(\F,c'\alpha/p)<\infty$, Corollary~\ref{cor:lp} gives total size
\[
O\!\left(\fat(\F,c'\alpha/p)\ln(2/\alpha)\ln^2(2p/\alpha)\right)
\]
and reconstructed loss at most $\alpha$ above the empirical infimum on every finite sample of positive length. The constants are universal, including in $p$. Section~\ref{sec:lp} states the extension and explains the adjusted scale; Appendix~\ref{app:lp} proves it.

\subsection{Proof Overview}\label{sec:proof-overview}
The construction uses the distinction between approximation at every sample position and control of average empirical loss. In the uniform-approximation construction of \citet{hanneke2019sample}, an ensemble is first made accurate at all sample positions and then sparsified using the dual class. For bounded average loss, a small fraction of inaccurate positions is permissible. This allowance makes our initial ensemble short enough to retain in full, removing the need for dual-class sparsification in the average-loss setting.

For the positive-dimensional branch, fix a sample $S$ and select an internal target $f^*\in\F$ whose empirical loss is at most $\alpha/8$ above the infimum. This target remains fixed throughout the construction, while the observed labels are left unchanged. The Lipschitz property of squared loss converts approximation of this target outside a sufficiently small exceptional set into the required agnostic guarantee (Section~\ref{sec:target}, Proposition~\ref{prop:target}). Under each current distribution on sample positions, we use the tolerant-learning bounds of \citet{hanneke2019sample} to obtain a small support on which approximate agreement implies weak accuracy. We quantize the target values on this support and apply a shared selector to the resulting constraints. Quantized auxiliary information also appears in the agnostic compression constructions of \citet{attias2024agnostic}. Here the constraints leave enough slack for the weak-error guarantee and determine the selected function through the shared rule (Section~\ref{sec:weak}, Proposition~\ref{prop:weak}).

We apply the median-boosting approach of \citet{kegl2003robust} with a fixed update coefficient chosen for this weak-error guarantee. The construction runs a prescribed $O(\ln(2/\alpha))$ number of rounds and forms their equal-weight upper median. Using the same coefficient in every round makes the exponential potential depend on the unweighted signed sum that controls this median. The potential bounds the fraction of sample positions where the median is inaccurate, so the round count depends only on accuracy, not on sample length (Proposition~\ref{prop:boost}). Every round is retained, including repeated functions. Section~\ref{sec:boost} gives the update rule and the key potential calculation.

A short ensemble still needs a finite description. We retain one original labeled observation for each distinct input in the union of the round supports, together with a target grid index and round-membership bits. These fields recover the same constraints at the decoder. Applying the same shared selection rule therefore recovers the same weak functions on the whole domain, without access to the internal target, the round distributions, or the original sample length. Retaining the round multiplicities then recovers exactly the same median (Proposition~\ref{prop:protocol}). Sections~\ref{sec:protocol}--\ref{sec:close} specify this protocol and count all retained observations and auxiliary bits. Together with the loss comparison, these guarantees prove Theorem~\ref{thm:main}. Appendix~\ref{app:map} maps the stages to their complete proofs and supporting lemmas.

\section{Related Work}\label{sec:related}
\stagelead{Compression, combinatorial dimension, and learning.}
\citet{littlestone1986relating} introduced sample compression as a connection between learning and short descriptions. For binary classes, \citet{moran2016sample} established compression of size exponential in the VC dimension. In real-valued learning, the relevant dimension must also reflect the prediction scale. The scale-sensitive theory of \citet{alon1997scale,bartlett1998prediction} explains this dependence through uniform convergence and learnability.

\citet{david2016supervised} developed a general conversion from supervised learning to compression. Their Theorem~4.3(2) gives approximate agnostic compression for bounded losses from a learner with a bounded sample requirement. Combining this conversion with scale-sensitive uniform convergence \citep{alon1997scale} gives a primal-only squared-loss bound of order $\widetilde O(\fat(\F,c\alpha)/\alpha^2)$. Here $c>0$ is an absolute constant. This comparison is a consequence of those results, rather than a bound stated directly in the conversion theorem. Our quantitative improvement is to replace the polynomial accuracy dependence by a polylogarithmic one while retaining linear dependence on the primal dimension.

\stagelead{Uniform regression compression through median boosting.}
\citet{hanneke2019sample} gave a general compression construction for realizable real-valued prediction with approximation at every sample position. Their method develops tolerant weak learners using scale-sensitive covering estimates, including the entropy bound of \citet{mendelson2003entropy}. It combines these learners through the median-boosting method of \citet{kegl2003robust}. A subsequent ensemble sparsification removes the dependence on sample length and yields a size bound involving both primal and dual fat-shattering dimensions.

The present construction uses their tolerant-learning inputs and the same underlying median-boosting approach. Its average-loss objective permits a small exceptional set, so the initial ensemble can already be short enough to retain in full. Equal-weight median aggregation also appears after sparsification in \citet{hanneke2019sample}. Our use of a fixed update coefficient makes equal weights available for the complete short ensemble and gives the finite protocol described in Section~\ref{sec:protocol}. The resulting agnostic guarantee concerns empirical average loss. The earlier uniform-approximation guarantee is stronger on realizable samples.

\stagelead{Agnostic and robust regression compression.}
\citet{attias2023adversarially} developed compression methods for adversarially robust real-valued prediction. Taking each perturbation set to be a singleton recovers the ordinary prediction setting. \citet{attias2024agnostic} studied agnostic $\ell_p$ compression and obtained general approximate schemes with primal--dual dependence. Their construction explicitly encodes quantized auxiliary values. For $1<p<\infty$, they also proved that bounded exact agnostic compression is impossible in general, already for constant regressors. This extends the range of losses covered by the squared-loss obstruction of \citet{david2016supervised}. Their squared-loss question asks for the stronger dependence on the primal dimension addressed here.

Our accuracy lower bound follows the counting approach used in these exact-compression obstructions. The proof of \citet[Theorem~4.6]{attias2024agnostic} counts codes for constant regressors on binary-labeled samples. We instead choose empirical means that remain distinguishable when excess squared loss $\alpha$ is allowed. This gives a logarithmic precision cost for approximate compression. For codes retaining original examples from the sample, it explains the negative answer to the log-free formulation, alongside the positive answer to the polylogarithmic preprint formulation.

\stagelead{Subsequent reductions and scale-sensitive objectives.}
\citet{attias2025reductions} studied reductions from regression compression to binary compression. Their Open Problem~17 concerns a conditional reduction for pointwise approximation in terms of fat-shattering dimension. Open Problem~20 concerns exact or accuracy-independent bounds in terms of pseudo-dimension. \citet{aiyer2026scale} studied learnability and evaluability at an optimal scale. Their Lemma~5 gives realizable pointwise compression with a dual factor and polynomial dependence on the approximation parameter. These pointwise and pseudo-dimension questions impose different requirements from near-optimal empirical squared loss.

The possible exponential separation between primal and dual dimensions \citep{kleer2021primal} explains the gain from removing that factor from the accuracy-efficient compression bound. Appendix~\ref{app:comparison} gives an explicit finite example of that separation. The example compares dimensions and the resulting upper bounds; it does not give a compression lower bound.

\section{Setting and Fixed Choices}\label{sec:setting}
For a positive integer $m$, let $[m]=\{1,\ldots,m\}$. Samples are sequences, so an input can occur more than once and its observed labels can differ. A distribution $P$ on $[m]$ assigns nonnegative masses summing to one; $P(A)=\sum_{i\in A}P(i)$ for $A\subseteq[m]$. We use natural logarithms unless a base is indicated explicitly.

\begin{definition}[Compression model]\label{def:compression}
For fixed $\F$ and $\alpha$, a compression map sends $S$ to a pair $(S',B)$, where $S'$ is a subsequence of the original labeled observations and $B\in\{0,1\}^{*}$. A reconstruction map sends a code to a function $\X\to[0,1]$. The total size of a code is $|S'|+|B|$. An $\alpha$-approximate agnostic scheme satisfies the loss inequality of Theorem~\ref{thm:main} for every sample, with a size bound independent of $m$.
\end{definition}

This uses the standard accounting of selected observations plus side-information bits \citep{david2016supervised,attias2024agnostic}. The maps are mathematical functions and need not be computable. Each stored original observation counts as one unit; no finite bit representation for every element of an arbitrary domain is assumed.

\begin{definition}[Fat-shattering dimension]\label{def:fat}
For $t>0$, distinct points $x_1,\ldots,x_k$ are $t$-fat-shattered by $\F$ if there are thresholds $r_1,\ldots,r_k$ such that, for every $\sigma\in\{-1,+1\}^k$, some $f_\sigma\in\F$ satisfies
$\sigma_i(f_\sigma(x_i)-r_i)\ge t$ for all $i\in[k]$.
The supremum of such integers $k$ is $\fat(\F,t)$, with value $0$ when no singleton is shattered. The dual class consists of the evaluation maps $x^*:f\mapsto f(x)$ on $\F$; its fat dimension is denoted $\fatd(\F,t)$.
\end{definition}

Fix a well-order of $\F$ and let $\Sel(H)$ denote the least element of each nonempty $H\subseteq\F$. Fix also a total order on $\X$. These choices are shared by the compression and reconstruction maps and are independent of the sample. All remaining ties are resolved as specified below. The use of such fixed orders is an existence construction.

\stagelead{Constants and their source.}
The covering estimate of \citet[Lemma~8]{hanneke2019sample} supplies universal constants $C_E,c_E>0$. Fix one valid pair. Appendix~\ref{app:weak} states that estimate and the companion probability bound precisely. It derives the tolerant-learning constants
\begin{equation}\label{eq:constants}
c_1=192\max\{C_E,1\},\qquad c_2=c_E/16,\qquad c_3=32.
\end{equation}
Put $C_2=\max\{1,c_2\}$ and fix
\begin{equation}\label{eq:parameters}
c'=\frac{c_2}{128C_2},\qquad
d=\fat(\F,c'\alpha),\qquad
\eta=\frac{\alpha}{12C_2},\qquad
\eps=\frac{\alpha^2}{64}.
\end{equation}
Here $\eta/2$ is the prediction tolerance, and $\eps$ is the allowed exceptional fraction. In particular $c'\le1/128$ and $\eta\le\alpha/12$. All numerical budgets introduced below are fixed from these parameters before reading $S$.

\section{The Construction, Stage by Stage}\label{sec:construction}
Average loss allows a small fraction of inaccurate sample positions. Section~\ref{sec:target} specifies this allowance after choosing an internal target. Section~\ref{sec:weak} constructs each weak function from quantized constraints at a scale proportional to the requested accuracy. The fixed-step updates in Section~\ref{sec:boost} reduce the inaccurate fraction using an ensemble whose length depends only on accuracy. We can therefore retain every round. Section~\ref{sec:protocol} encodes the constraints and recovers the same functions and their equal-weight median. Section~\ref{sec:close} combines this reconstruction with the loss guarantee and counts the complete description.

We first separate the case $d=0$. The compressor stores nothing, and the decoder outputs the fixed function $f_0=\Sel(\F)$ on every code. Its loss is at most $\opt(S)+\alpha/32$ for every sample, as shown in Appendix~\ref{app:zero}. All subsequent stages assume $1\le d<\infty$.

\subsection{Choose an Internal Comparison Target}\label{sec:target}
Given $S$, select
\begin{equation}\label{eq:target}
f^*=\Sel\{f\in\F:L_2(f,S)\le\opt(S)+\alpha/8\}.
\end{equation}
Use $f^*(x_i)$ as the target at index $i$ in every round. These values are internal to compression. The observed labels $y_i$ remain unchanged. This supplies a consistent function to approximate even when repeated inputs have conflicting labels.

\begin{proposition}[Target and final loss interface]\label{prop:target}
The set in \eqref{eq:target} is nonempty. For the selected $f^*$, if $g:\X\to[0,1]$ and $G\subseteq[m]$ satisfy
\[
|[m]\setminus G|\le\eps m,
\qquad |g(x_i)-f^*(x_i)|\le\eta/2\quad(i\in G),
\]
then $L_2(g,S)\le\opt(S)+\alpha/8+\eta+\eps<\opt(S)+\alpha$.
\end{proposition}
The allowance for exceptional positions comes from averaging the loss. Squared loss on $[0,1]$ is $2$-Lipschitz in its prediction and lies in $[0,1]$. Under the proposition's hypotheses,
\[
L_2(g,S)-L_2(f^*,S)
\le\eta\frac{|G|}{m}+\frac{|[m]\setminus G|}{m}
\le\eta+\eps.
\]
The target's $\alpha/8$ slack then gives the conclusion. Appendices~\ref{app:loss}--\ref{app:target-proof} give the full argument. The next stage supplies one weak function; Section~\ref{sec:boost} amplifies its guarantee to obtain the required good set.

\subsection{Construct One Weak Function from a Short Description}\label{sec:weak}
The input to this stage is a distribution $P$ on $[m]$, together with $S$ and the fixed target $f^*$. Define the tuple length and target grid by
\begin{equation}\label{eq:n-grid}
n=\left\lceil\frac{8c_1}{3}\left(d\ln\frac{32c_3}{3\eta}+\ln2\right)\right\rceil,
\qquad J=\lceil16/\eta\rceil,\qquad b_q=\lceil\log_2(J+1)\rceil.
\end{equation}
A tuple $I=(I_1,\ldots,I_n)\in[m]^n$ is \emph{good for $P$} when
\begin{equation}\label{eq:good-tuple}
\forall f\in\F,\qquad
\max_{\ell\in[n]}|f(x_{I_\ell})-f^*(x_{I_\ell})|\le\eta/4
\ \Longrightarrow\
P\{i:|f(x_i)-f^*(x_i)|>\eta/2\}\le3/8.
\end{equation}
Choose the lexicographically least good tuple and let $V=\{x_{I_\ell}:\ell\in[n]\}$ be its set of distinct inputs. For each $x\in V$, choose $j_x\in\{0,\ldots,J\}$ minimizing $|j_x/J-f^*(x)|$, with ties to the smaller integer, and put $q_x=j_x/J$. Define
\begin{equation}\label{eq:selection}
H(V,q)=\{f\in\F:|f(x)-q_x|\le\eta/16\text{ for every }x\in V\},
\qquad h=\Sel(H(V,q)).
\end{equation}
The same input receives the same grid index in every round. Once $V$ and its grid indices are known, the selection of $h$ uses only the shared class and selection rule.

\begin{proposition}[A short description of a weak function]\label{prop:weak}
For every distribution $P$ on $[m]$, a good tuple exists. The set $H(V,q)$ is nonempty, and its selected function satisfies
\[
P\{i:|h(x_i)-f^*(x_i)|>\eta/2\}\le3/8.
\]
The whole function $h$ is determined by $V$, the grid indices $(j_x)_{x\in V}$, and the fixed selection rule. It is unchanged by changes in the tuple's order or multiplicities that preserve $V$. One original representative per input in $V$, together with its $b_q$-bit grid index, uses at most $n$ observations and $nb_q$ bits, before round-membership fields are added.
\end{proposition}

Quantization leaves enough slack for the weak guarantee. Nearest-grid rounding gives $|q_x-f^*(x)|\le\eta/32$, so $f^*\in H(V,q)$. Every selected $h$ satisfies
\[
|h(x)-f^*(x)|\le\frac\eta{16}+\frac\eta{32}
=\frac{3\eta}{32}<\frac\eta4\qquad(x\in V).
\]
Since $V$ comes from a good tuple, \eqref{eq:good-tuple} now gives the error bound under $P$. The tolerant-sampling argument uses fat-shattering scale $3c_2\eta/32=c'\alpha$, preserving the linear accuracy scale of Theorem~\ref{thm:main}. Appendices~\ref{app:tuple}--\ref{app:weak-proof} establish these claims from the tolerant-learning inputs of \citet{hanneke2019sample}. Once the constraints are fixed, the shared selector determines the whole function. Section~\ref{sec:protocol} uses this fact to recover each round without its distribution or the target.

\subsection{Update the Distribution and Form an Equal-Weight Median}\label{sec:boost}
Start with $P_1(i)=1/m$ and fix
\begin{equation}\label{eq:round-count}
a_0=\frac12\ln\frac53,
\qquad \lambda_0=\frac{\sqrt{15}}4<1,
\qquad T=\left\lceil\frac{\ln(1/\eps)}{-\ln\lambda_0}\right\rceil.
\end{equation}
Here $a_0$ is the update coefficient, and $\lambda_0$ is the uniform bound on the normalizing factors stated below. Both are chosen for the weak-error bound $3/8$.

At round $t\in[T]$, apply the preceding stage to $P_t$ to obtain $V_t$ and $h_t$. Set
\[
E_t=\{i:|h_t(x_i)-f^*(x_i)|>\eta/2\},\qquad
e_t=P_t(E_t),\qquad \theta_i^{(t)}=1-2\cdot\ind{i\in E_t},
\]
and update
\begin{equation}\label{eq:update}
Z_t=\sum_{i=1}^mP_t(i)e^{-a_0\theta_i^{(t)}},
\qquad P_{t+1}(i)=\frac{P_t(i)e^{-a_0\theta_i^{(t)}}}{Z_t}.
\end{equation}
Each update increases the relative mass of the inaccurate positions. Run exactly $T$ rounds, including when $e_t=0$, in which case the distribution remains unchanged. The target $f^*$ stays fixed throughout.

After the rounds, define the compressor's predictor by
\begin{equation}\label{eq:equal-median}
g(x)=\min\left\{h_j(x):j\in[T],\
\sum_{t=1}^T\ind{h_t(x)>h_j(x)}<\frac T2\right\}.
\end{equation}
This is the upper median of the $T$ predictions at $x$, namely the value in sorted position $\lfloor T/2\rfloor+1$. Every round contributes once, including repeated functions. The distributions $P_t$ weight sample positions during construction; all rounds have equal weight in the final median.

\begin{proposition}[Fixed-step boosting with an equal-weight median]\label{prop:boost}
For every sample, all $T$ rounds are well-defined. Every $P_t$, $1\le t\le T+1$, has positive mass at every index, and $0<Z_t\le\lambda_0$ for $t\in[T]$. The function $g$ in \eqref{eq:equal-median} is well-defined and takes values in $[0,1]$. There is $G\subseteq[m]$ such that
\[
|[m]\setminus G|\le\eps m,
\qquad |g(x_i)-f^*(x_i)|\le\eta/2\quad(i\in G).
\]
The round budget is $T=O(\ln(2/\alpha))$, independently of $m$.
\end{proposition}
\stagelead{Why the full ensemble is short.}
Let $\mathcal E_g=\{i:|g(x_i)-f^*(x_i)|>\eta/2\}$ and $D_i=\sum_{t=1}^T\theta_i^{(t)}$. A strict majority of accurate rounds keeps the upper median in the target interval. Thus $i\in\mathcal E_g$ implies $D_i\le0$, and $\ind{i\in\mathcal E_g}\le e^{-a_0D_i}$. Proposition~\ref{prop:weak} gives $e_t\le3/8$, so
$Z_t=(1-e_t)e^{-a_0}+e_te^{a_0}\le\lambda_0$.
Unrolling \eqref{eq:update} from $P_1(i)=1/m$ gives
\[
P_{T+1}(i)=\frac{e^{-a_0D_i}}{m\prod_{t=1}^T Z_t}.
\]
Summing this probability distribution yields
\[
\frac{|\mathcal E_g|}{m}\le\frac1m\sum_{i=1}^m e^{-a_0D_i}
=\prod_{t=1}^T Z_t\le\lambda_0^T\le\eps.
\]
The common update coefficient makes this potential depend on the same unweighted signed sum that controls the median. To certify zero inaccurate positions from this bound, one would need $\lambda_0^T<1/m$. Here reaching $\eps$ suffices, so the chosen round count depends only on accuracy. This is the use of average loss that permits retaining all rounds. The argument follows the median-boosting approach of \citet{kegl2003robust,hanneke2019sample}. Appendices~\ref{app:fixed-update}--\ref{app:boost-proof} give the supporting lemmas and complete proof. The next stage encodes the entire ensemble.

\subsection{Store the Inputs and Reconstruct the Same Function}\label{sec:protocol}
Let $V=\bigcup_{t=1}^T V_t$, set $U=|V|$, and list the distinct inputs as $\xi_1\prec\cdots\prec\xi_U$ in the shared order on $\X$. For each input choose its first occurrence $i_x=\min\{i:x_i=x\}$. Sort these indices as $i_{(1)}<\cdots<i_{(U)}$ and store the original subsequence
\begin{equation}\label{eq:stored-subsequence}
S'=((x_{i_{(r)}},y_{i_{(r)}}))_{r=1}^U.
\end{equation}
The decoder sorts the stored inputs to recover the order $(\xi_r)$ used for the bit fields. Original labels are retained in $S'$; target grid values are stored in the auxiliary string.

For $r\in[U]$ and $t\in[T]$, set $c_{r,t}=\ind{\xi_r\in V_t}$ and concatenate
\begin{equation}\label{eq:code-layout}
B=\bigl(\operatorname{bin}_{b_q}(j_{\xi_r}),c_{r,1},\ldots,c_{r,T}\bigr)_{r=1}^U.
\end{equation}
Here $\operatorname{bin}_{b_q}$ is the unsigned $b_q$-bit binary representation, with leading zeros and most significant bit first. Each block contains a grid index and $T$ membership bits. The decoder knows $J,T,b_q$ and obtains $U$ from $S'$. The membership columns preserve all rounds, even when two rounds use the same inputs or select the same function.

After reading $q_{\xi_r}=j_{\xi_r}/J$, the decoder forms
\begin{equation}\label{eq:round-reconstruction}
V_t^{\mathrm{rec}}=\{\xi_r:c_{r,t}=1\},\qquad
h_t^{\mathrm{rec}}=\Sel(H(V_t^{\mathrm{rec}},q)),
\end{equation}
and outputs \eqref{eq:equal-median} with $h_t^{\mathrm{rec}}$ in place of $h_t$. Denote this output by $\hat f$. Appendix~\ref{app:decode} completes this rule on every code, using zero for malformed fields or empty required sets.

\begin{proposition}[Exact reconstruction from the finite code]\label{prop:protocol}
For every sample, the described code is valid and uses original observations. The reconstruction rule is a total deterministic map on codes. On every code produced by the compressor,
\[
V_t^{\mathrm{rec}}=V_t,\qquad h_t^{\mathrm{rec}}=h_t\quad(t\in[T]),
\qquad \hat f(x)=g(x)\quad(x\in\X).
\]
It requires neither $f^*$, the unsaved observations, the distributions $P_t$, nor the sample length $m$.
\end{proposition}
On a compressor-produced code, the membership fields recover $V_t$ and the grid fields recover the same $q_x$. Both maps therefore use the same nonempty constraint set:
\[
h_t^{\mathrm{rec}}=\Sel(H(V_t^{\mathrm{rec}},q))
=\Sel(H(V_t,q))=h_t.
\]
This is equality of functions on $\X$; the feasible set need not be a singleton. Repeated rounds remain separate entries, so the medians agree everywhere as well. Decoding introduces no further discrepancy between the constructed and reconstructed predictors. Appendix~\ref{app:protocol-proof} verifies the fields and proves the complete guarantee.

\subsection{Collect the Guarantees and Count the Description}\label{sec:close}
Proposition~\ref{prop:protocol} identifies $\hat f$ with $g$. Proposition~\ref{prop:boost} supplies the good set required by Proposition~\ref{prop:target}, which gives the loss bound against the original labels. The remaining guarantee accounts for the full transmitted description.

\begin{proposition}[Total description size]\label{prop:size}
For $d\ge1$, the code satisfies $|S'|=U\le Tn$ and $|B|=U(b_q+T)$. Hence
\begin{equation}\label{eq:total-size}
|S'|+|B|=U(1+b_q+T)\le Tn(1+b_q+T)
=O(d\ln^3(2/\alpha)).
\end{equation}
Separately, $|S'|=O(d\ln^2(2/\alpha))$ and $|B|=O(d\ln^3(2/\alpha))$. All implied constants are universal. The branch $d=0$ uses total size zero.
\end{proposition}
For $d\ge1$, write $L=\ln(2/\alpha)$. The single-round support budget is $n=O(dL)$, the round count is $T=O(L)$, and the grid-index field uses $b_q=O(L)$ bits. Thus $U\le Tn=O(dL^2)$, with $b_q+T=O(L)$ auxiliary bits per retained input. The grid and membership fields contribute additively to this width. These are the three logarithmic factors in \eqref{eq:total-size}. No original sample index or aggregation coefficient is transmitted. Appendix~\ref{app:size} checks the exact counts, and Appendix~\ref{app:closure} proves Theorem~\ref{thm:main} by combining them with the reconstructed loss guarantee.

\section{An Unavoidable Dependence on Accuracy}\label{sec:accuracy-lower}
The dependence on accuracy cannot be removed entirely when original observations and auxiliary bits are charged together. This remains true for a class whose fat-shattering dimension is one at every sufficiently small scale. The following lower bound uses the original-observation compression model of Definition~\ref{def:compression}.

\begin{theorem}[An accuracy lower bound]\label{thm:accuracy-lower}
Let $\X_0=\{x_0\}$ and let
\[
\mathcal C=\{f_u:\X_0\to[0,1]:f_u(x_0)=u,\ u\in[0,1]\}.
\]
Then $\fat(\mathcal C,t)=1$ for every $0<t\le1/2$. Fix $0<\alpha<1/4$. Suppose that a deterministic $\alpha$-approximate agnostic squared-loss compression scheme for $\mathcal C$ has total size at most an integer $k\ge0$ on every finite sample. Then
\[
k\ge\frac14\log_2\frac1\alpha-1.
\]
The compression and reconstruction maps may depend on $\alpha$. The same bound holds if $S'$ may be any finite sequence of labeled examples occurring in the sample, in any order and with any multiplicities, with each occurrence counted toward $|S'|$. This includes a retained submultiset $\widehat S\subseteq S$.
\end{theorem}

\stagelead{The precision cost of the empirical mean.}
For a binary-labeled sample, set $\bar y=m^{-1}\sum_i y_i$ and $\opt_{\mathcal C}(S)=\inf_{f\in\mathcal C}L_2(f,S)$. With $f_z(x_0)=z\in[0,1]$,
\[
L_2(f_z,S)-\opt_{\mathcal C}(S)=(z-\bar y)^2.
\]
Thus an excess loss of at most $\alpha$ requires $|z-\bar y|\le\sqrt\alpha$. At a common sample length, one can choose at least $1/(2\sqrt\alpha)$ empirical means with pairwise gaps greater than $2\sqrt\alpha$. Their codes must be distinct. Each retained observation has only two possible values, so descriptions with $r$ observations and $b$ bits number at most $2^{r+b}$. Summing over nonnegative integers $r,b$ gives
\[
\#\mathrm{codes}\le\sum_{r+b\le k}2^{r+b}\le2^{2k+1}.
\]
Comparison with the number of separated means gives the stated bound. Appendix~\ref{app:accuracy-lower-proof} supplies the common length, separation, and exact count. Using a common length also shows that supplying it to the decoder does not remove the obstruction.

\begin{corollary}[No universal bound by the dimension alone]\label{cor:no-log}
There are no universal constants $C,c>0$ with the following property: for every nonempty class $\F\subseteq[0,1]^{\X}$ on an arbitrary domain $\X$, and every $0<\alpha\le1$ with $\fat(\F,c\alpha)<\infty$, there is a deterministic $\alpha$-approximate agnostic squared-loss compression scheme, in the model of Definition~\ref{def:compression}, whose total size is at most
\[
C\fat(\F,c\alpha)
\]
on every finite sample.
\end{corollary}

Appendix~\ref{app:no-log-proof} derives the corollary from Theorem~\ref{thm:accuracy-lower}. The lower-bound proof follows the counting approach used for exact agnostic compression lower bounds by \citet[Theorem~4.1]{david2016supervised} and \citet[Theorem~4.6]{attias2024agnostic}. In particular, the latter argument already uses binary labels in constant regression. Here a finite collection of separated empirical means gives an obstruction at positive approximation error.

Corollary~\ref{cor:no-log} rules out the literal log-free target in \citet[Section~5]{attias2024agnostic} under that paper's accounting of retained original observations plus auxiliary bits.\footnote{Their Section~2 charges $k'+k''\le k$ for observations and bits. Although its map signature does not state the original-sample restriction explicitly, the proof of their Theorem~4.6 uses a retained multiset $\widehat S\subseteq S$. For squared loss on the singleton domain, allowing synthetic labeled examples would permit an exact size-one code storing $(x_0,\bar y)$, where $\bar y$ is the empirical label mean. Thus the $p=2$ case of their Theorem~4.6 relies on the original-example restriction. Our scheme retains original observations, and its decoder sorts their inputs (Appendix~\ref{app:decode}). It therefore also works with a retained submultiset.} The lower bound does not establish that the cubic logarithm in Theorem~\ref{thm:main} is necessary, or that an accuracy-dependent term must multiply the dimension. It leaves open the optimal dependence on accuracy for general classes.

\section{Extension to Average Power Loss}\label{sec:lp}
For $1\le p<\infty$, let $L_p(f,S)=m^{-1}\sum_i|f(x_i)-y_i|^p$ and $\opt_p(S)=\inf_{f\in\F}L_p(f,S)$. This is the average $p$-th power, without taking a root.
\begin{corollary}\label{cor:lp}
For $\varnothing\ne\F\subseteq[0,1]^{\X}$, $0<\alpha\le1$, and $1\le p<\infty$, suppose $d_p=\fat(\F,c'\alpha/p)<\infty$. There is a deterministic compression scheme of size
\[
C d_p\ln(2/\alpha)\ln^2(2p/\alpha)
\]
with $L_p(\hat f,S)\le\opt_p(S)+\alpha$ on every sample. The constant $C$ is universal, including in $p$, and the code is empty if $d_p=0$.
\end{corollary}
Choose the internal target using $L_p$, and replace $\eta$ by $\alpha/(12pC_2)$. All later stages are unchanged except for their scale-dependent parameters. The loss transfer now uses the $p$-Lipschitz property of $u\mapsto|u-y|^p$ on $[0,1]$. Appendix~\ref{app:lp} proves the result, including the distinct finiteness assumption at scale $\alpha/p$.

\section{Conclusion}
\label{sec:conclusion}
The construction gives a positive answer, for bounded regression, to the polylogarithmic formulation of the squared-loss compression question posed by \citet[Section~5]{attias2024agnostic_arxiv}. The total description is linear in the primal fat-shattering dimension at scale proportional to $\alpha$, with a factor $\ln^3(2/\alpha)$. It covers arbitrary observed labels and counts both original observations and auxiliary bits.

The decisive use of the average-loss objective is that a small exceptional fraction is sufficient. Fixed-step median boosting then produces an ensemble whose length depends only on accuracy. Quantized target constraints and round-membership bits allow the decoder to recover every weak function and the same equal-weight median on the whole domain. These approximation and reconstruction guarantees together give the bound. The same construction applies to average power loss after adjusting the target and prediction scale.

Theorem~\ref{thm:accuracy-lower} shows that some additional accuracy dependence is unavoidable even at fat dimension one. Corollary~\ref{cor:no-log} therefore rules out the literal log-free target in the proceedings version \citep[Section~5]{attias2024agnostic}, when retained observations are original examples from the sample and observations and auxiliary bits are charged together (Definition~\ref{def:compression}). The optimal dependence between this logarithmic lower bound and the cubic logarithm in our upper bound remains open. Further questions concern efficient implementations and comparable bounds for the stronger objective of approximation at every sample position.

\bibliography{references/references}
\clearpage
\appendix
\section{Reading the Proofs and Their Dependencies}\label{app:map}
The main text follows the five stages of the construction and explains their key mechanisms locally. The appendices give complete proofs in dependency order, including the steps summarized in the main text. Within each group, the supporting lemmas precede the direct stage guarantee. Before proving a direct guarantee, we identify the stage it supports and the earlier results used in its proof.

Appendix~\ref{app:weak} begins with the loss comparison, target selection, and zero-dimensional case. It then derives the weak-function guarantee from tolerant sampling and quantized constraints. The two external inputs are the finite-domain forms of HKS Theorem~7 and Lemma~8. Appendix~\ref{app:external} states these inputs precisely and accepts them from that source. The subsequent sampling and selection arguments are proved here.

Appendix~\ref{app:boost} follows the update and aggregation stage. Its supporting lemmas establish validity and contraction of the fixed update (Appendix~\ref{app:fixed-update}), the interval property of the upper median (Appendix~\ref{app:median}), and the exponential potential for median errors (Appendix~\ref{app:potential}). Their combination proves Proposition~\ref{prop:boost}.

Appendix~\ref{app:protocol} first defines the decoder on every code. It proves that the stored fields recover exactly the compressor's function, independently of that function's accuracy. The size calculation then accounts for every observation and bit. The proof of Theorem~\ref{thm:main} combines exact reconstruction, the boosting guarantee, the loss comparison, and the size bound. The zero-dimensional branch enters this final proof separately.

The remaining two appendices have different roles. Appendix~\ref{app:accuracy-lower} proves the accuracy lower bound and its consequence for the log-free target. This is an independent argument based on separated empirical means and finite-code counting. Appendix~\ref{app:extensions} reuses the construction at a different scale for power loss and verifies its loss and size bounds. Its final dimension-gap example supports only the related-work comparison.

\begin{table}[!htbp]
\centering
\begin{tabular}{@{}>{\raggedright\arraybackslash}p{.19\textwidth}>{\raggedright\arraybackslash}p{.39\textwidth}>{\raggedright\arraybackslash}p{.34\textwidth}@{}}
\toprule
Part & Purpose & Guarantee and proof location\\
\midrule
Appendix~\ref{app:weak} & Loss interface, tolerant tuples, quantized selection & Proposition~\ref{prop:target}: \ref{app:target-proof}; Proposition~\ref{prop:weak}: \ref{app:weak-proof}\\
Appendix~\ref{app:boost} & Fixed updates, upper median, error potential & Proposition~\ref{prop:boost}: \ref{app:boost-proof}\\
Appendix~\ref{app:protocol} & Exact reconstruction and complete accounting & Proposition~\ref{prop:protocol}: \ref{app:protocol-proof}; Proposition~\ref{prop:size}: \ref{app:size}; Theorem~\ref{thm:main}: \ref{app:closure}\\
Appendix~\ref{app:accuracy-lower} & Accuracy dependence and the log-free target & Theorem~\ref{thm:accuracy-lower}: \ref{app:accuracy-lower-proof}; Corollary~\ref{cor:no-log}: \ref{app:no-log-proof}\\
Appendix~\ref{app:extensions} & Power loss and a separate dimension comparison & Corollary~\ref{cor:lp}: \ref{app:lp}; Section~\ref{sec:related}: \ref{app:comparison}\\
\bottomrule
\end{tabular}
\caption{Proof Roadmap.}
\label{tab:proof-roadmap}
\end{table}
Figure~\ref{fig:dependencies} records proof dependencies, not the successive states of the boosting procedure. The dashed arrow to the power-loss corollary denotes reuse of the construction and proof at a new scale, rather than an application of the squared-loss conclusion.
\clearpage
\begin{figure}[p]
\centering
\includegraphics[width=\textwidth,height=.82\textheight,keepaspectratio]{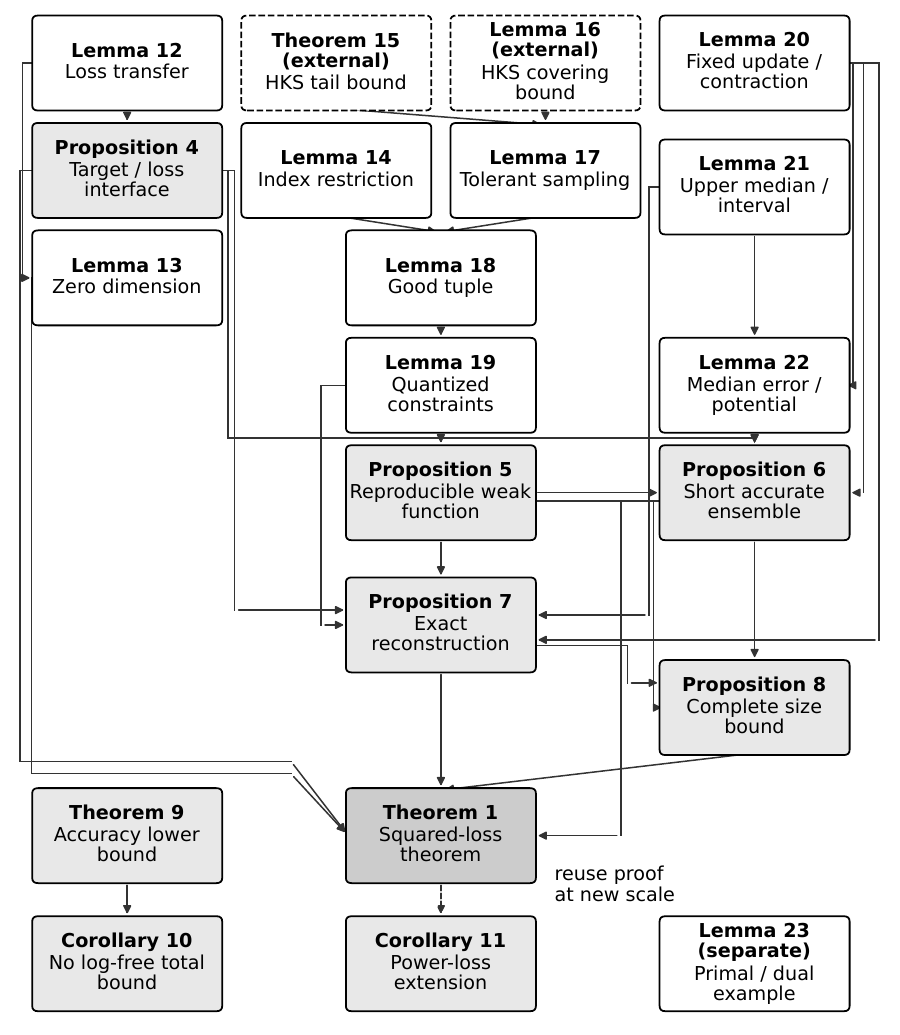}
\caption{Proof dependencies for the fixed-step construction. Shaded nodes are main-text guarantees; dashed boxes are accepted external inputs. Solid arrows indicate use of an upstream result, with some transitive arrows omitted. The dashed arrow indicates proof reuse at a different scale. The accuracy lower bound and its corollary form an independent branch at the lower left.}
\label{fig:dependencies}
\end{figure}
\clearpage

\section{From the Comparison Target to a Reproducible Weak Function}
\label{app:weak}

This appendix proves the guarantees used in Sections~\ref{sec:target}
and~\ref{sec:weak}. We first justify the choice of the internal target and
the conversion from approximation of that target to empirical squared
loss. We then state the two external results used to obtain a good tuple.
The remaining arguments show that quantized values on this tuple determine
a weak function through the fixed selection rule.

Throughout, $\F\subseteq[0,1]^{\X}$ is nonempty and
$S=((x_i,y_i))_{i=1}^m$ has $m\ge1$. The sequence may contain repeated
inputs and conflicting labels. Neither the loss estimates nor the
restriction arguments assume a probability distribution on $S$.

\subsection{Loss Transfer and Selection of the Target}
\label{app:loss}

The next lemma provides the loss estimate in Proposition~\ref{prop:target}.
It will also close the proof of the main theorem after the reconstruction
has been shown to approximate the internal target at most sample
positions.

\begin{lemma}[Squared-loss transfer]
\label{lem:loss}
For $u,v,y\in[0,1]$,
\begin{equation}
\label{eq:app-point-loss}
 \bigl|(u-y)^2-(v-y)^2\bigr|
 \le \min\{2|u-v|,1\}.
\end{equation}
Consequently, let $g,h:\X\to[0,1]$, let $r\ge0$, and let
$G\subseteq[m]$. If $|g(x_i)-h(x_i)|\le r$ for every $i\in G$, then
\begin{equation}
\label{eq:app-loss-transfer}
 L_2(g,S)-L_2(h,S)
 \le 2r\frac{|G|}{m}+\frac{|[m]\setminus G|}{m}.
\end{equation}
In particular, if $|[m]\setminus G|/m\le\varepsilon_0$, the right-hand
side is at most $2r+\varepsilon_0$.
\end{lemma}

\begin{proof}
Factoring the difference of squares gives
\[
 \bigl|(u-y)^2-(v-y)^2\bigr|
 =|u-v|\,|u+v-2y|.
\]
The range assumptions imply $-2\le u+v-2y\le2$, so this expression
is at most $2|u-v|$. Both squared losses belong to $[0,1]$.
Their absolute difference is therefore also at most $1$.
This proves~\eqref{eq:app-point-loss}.

Write $G^c=[m]\setminus G$. At a position in $G$, the difference of
the losses is at most $2r$ by~\eqref{eq:app-point-loss}. At a position
in $G^c$, the same difference is at most $1$. Thus
\begin{align*}
 L_2(g,S)-L_2(h,S)
 &=\frac1m\sum_{i\in G}
       \bigl((g(x_i)-y_i)^2-(h(x_i)-y_i)^2\bigr)\\
 &\quad+\frac1m\sum_{i\in G^c}
       \bigl((g(x_i)-y_i)^2-(h(x_i)-y_i)^2\bigr)\\
 &\le 2r\frac{|G|}{m}+\frac{|G^c|}{m}.
\end{align*}
Finally, $|G|/m\le1$. This yields the last assertion.
The argument does not require $g$ or $h$ to belong to $\F$.
\end{proof}

\subsection{Proof of Proposition~\ref{prop:target}}
\label{app:target-proof}

Proposition~\ref{prop:target} supports the first construction stage,
Section~\ref{sec:target}. It provides a well-defined comparison target
and specifies what approximation of that target will suffice at the end
of the construction. Its loss conclusion uses Lemma~\ref{lem:loss}.

\begin{proof}
The set of attainable losses
\[
 \mathcal L_S=\{L_2(f,S):f\in\F\}
\]
is nonempty and is contained in $[0,1]$. It therefore has a real
infimum, namely $\opt(S)$. For every $\tau>0$, some $f\in\F$
satisfies $L_2(f,S)<\opt(S)+\tau$. Indeed, if no such $f$ existed,
$\opt(S)+\tau$ would be a lower bound for $\mathcal L_S$ strictly
larger than its greatest lower bound.

Apply this fact with $\tau=\alpha/8$. The set
\[
 \F_S^*=
 \{f\in\F:L_2(f,S)\le\opt(S)+\alpha/8\}
\]
is nonempty. The fixed selection rule consequently defines
$f^*=\Sel(\F_S^*)$. In particular,
\begin{equation}
\label{eq:app-near-minimizer}
 L_2(f^*,S)\le\opt(S)+\frac\alpha8.
\end{equation}
This selection is deterministic once $S$, $\F$, $\alpha$, and the
fixed rule are given. It does not assume that the infimum is attained.

Let $g:\X\to[0,1]$ satisfy $|g(x_i)-f^*(x_i)|\le r$ on a set
$G\subseteq[m]$ with $|[m]\setminus G|/m\le\varepsilon_0$.
Lemma~\ref{lem:loss}, followed by~\eqref{eq:app-near-minimizer}, gives
\begin{equation}
\label{eq:app-target-interface}
 L_2(g,S)
 \le\opt(S)+\frac\alpha8+2r+\varepsilon_0.
\end{equation}
For the parameters used by the construction, $r=\eta/2$ and
$\varepsilon_0=\eps=\alpha^2/64$. Since
$\eta=\alpha/(12C_2)$, $C_2\ge1$, and $0<\alpha\le1$,
\[
 \frac\alpha8+\eta+\eps
 \le\left(\frac18+\frac1{12}+\frac1{64}\right)\alpha
 =\frac{43}{192}\alpha<\alpha.
\]
Thus this approximation guarantee suffices for the required empirical
loss bound. The target values $f^*(x_i)$ are internal comparison values.
No step of the argument identifies them with the original labels $y_i$.
\end{proof}

\subsection{The Zero-Dimensional Branch}
\label{app:zero}

The following lemma justifies the empty code in Theorem~\ref{thm:main}.
It separates this branch before any estimate that absorbs an additive
constant into a positive dimension. It supplies the initial case of
Section~\ref{sec:construction}; its pointwise diameter conclusion is also
used for power losses in Appendix~\ref{app:lp}.

\begin{lemma}[Zero fat dimension permits an empty code]
\label{lem:zero}
If $\fat(\F,t)=0$ for $t>0$, then
\[
 |f(x)-g(x)|<2t
 \qquad(f,g\in\F,\ x\in\X).
\]
For any fixed $f_0\in\F$ and any sample $S$,
\begin{equation}
\label{eq:app-zero-loss}
 0\le L_2(f_0,S)-\opt(S)\le4t.
\end{equation}
In particular, if $d=\fat(\F,c'\alpha)=0$ and $c'\le1/128$,
the scheme that always stores the empty sample and the empty bit string,
and always reconstructs $f_0$, has size zero and excess empirical squared
loss at most $\alpha/32$.
\end{lemma}

\begin{proof}
Suppose that $|f(x)-g(x)|\ge2t$ for some $f,g\in\F$ and
$x\in\X$. After exchanging $f$ and $g$ if necessary, assume
$f(x)\ge g(x)+2t$. Set $r=(f(x)+g(x))/2$. Then
\[
 f(x)\ge r+t,
 \qquad
 g(x)\le r-t.
\]
The two functions realize both signs at the single point $x$ with
threshold $r$. This point is therefore $t$-fat-shattered. The conclusion
contradicts $\fat(\F,t)=0$. This proves the strict pointwise bound.

For every $f\in\F$, Lemma~\ref{lem:loss} now gives
\[
 L_2(f_0,S)\le L_2(f,S)+4t.
\]
In particular, $L_2(f_0,S)-4t$ is a lower bound for
$\{L_2(f,S):f\in\F\}$. It cannot exceed the infimum of this set.
Hence $L_2(f_0,S)\le\opt(S)+4t$. The other inequality
in~\eqref{eq:app-zero-loss} follows from $f_0\in\F$.

Take $t=c'\alpha$. Fix $f_0=\Sel(\F)$ before seeing the sample.
For every $S$, let the compression map output
$(S_{\varnothing},B_{\varnothing})$, where both components are empty.
Define the reconstruction map to output $f_0$ on every encoding.
These are deterministic maps. Their size on every compressed sample is
$0+0=0$, and their loss is at most
\[
 \opt(S)+4c'\alpha\le\opt(S)+\frac\alpha{32}.
\]
The fixed function is part of the reconstruction map, not information
transmitted for a particular sample. This branch requires no mode bit:
whether $d=0$ is fixed by $\F$ and $\alpha$, which specify the scheme.
\end{proof}

\subsection{Restriction to the Finite Index Domain}
\label{app:restriction}

To prove Proposition~\ref{prop:weak}, we apply the external sampling
results to sample positions. The next lemma establishes the dimension
bound needed for that application, including when several positions have
the same input.

\begin{lemma}[Scale monotonicity and restriction]
\label{lem:restriction}
If $0<s\le t$, then $\fat(\F,t)\le\fat(\F,s)$. For a fixed
input sequence $(x_i)_{i=1}^m$, define
\[
 \GS=\{g_f:[m]\to[0,1]:g_f(i)=f(x_i),\ f\in\F\}.
\]
For every $t>0$,
\begin{equation}
\label{eq:app-restriction}
 \fat(\GS,t)\le\fat(\F,t).
\end{equation}
\end{lemma}

\begin{proof}
Every set of points and thresholds witnessing $t$-fat-shattering also
witnesses $s$-fat-shattering: replacing $t$ by $s\le t$ weakens each
required inequality. This proves scale monotonicity.

Suppose $A\subseteq[m]$ is $t$-fat-shattered by $\GS$, with
thresholds $(r_i)_{i\in A}$. We first show that its inputs are distinct.
Assume to the contrary that $i,j\in A$, $i\ne j$, and $x_i=x_j$.
The sign pattern that is positive at $i$ and negative at $j$ gives a
function $f\in\F$ for which
\[
 f(x_i)\ge r_i+t,
 \qquad
 f(x_j)\le r_j-t.
\]
Since $f(x_i)=f(x_j)$, this implies $r_j-r_i\ge2t$.
The opposite pattern gives a possibly different function and implies
$r_i-r_j\ge2t$. Adding these two inequalities yields $0\ge4t$,
contrary to $t>0$.

It follows that $\{x_i:i\in A\}$ has cardinality $|A|$. Every sign
pattern on this set corresponds to one on $A$. A function in $\GS$
realizing the latter is the restriction of some function in $\F$.
That original function realizes the former with the same thresholds.
Thus $\F$ $t$-fat-shatters $\{x_i:i\in A\}$, and
$|A|\le\fat(\F,t)$. This proves~\eqref{eq:app-restriction}.
\end{proof}

The class $\GS$ need not be finite. Its domain is finite, however.
When a distribution $P$ is placed on $[m]$, every event in a finite
product sample space $[m]^q$ is measurable. This remains true for events
defined by quantifying over all functions in $\GS$. We use this fact
when applying the following external results. No measurability assertion
on an arbitrary domain $\X$ is needed.

\subsection{The External Sampling and Covering Results}
\label{app:external}

We use two results from Section~6 of
\citet{hanneke2019sample}: Theorem~7 on page~8 and Lemma~8 on page~10
of the published PMLR paper. We state their finite-domain specializations
in our notation. Both statements are accepted as external inputs.
The derivation after them, including the choice of constants, is given
here in full. Lemma~8 of that paper derives its covering estimate from
the entropy bound of \citet{mendelson2003entropy}.

Let $\Omega$ be a nonempty finite set, let
$\mathcal H\subseteq[0,1]^\Omega$ be nonempty, and let $P$ be a
probability distribution on $\Omega$. For $Z=(Z_1,\ldots,Z_{2q})$
and $r\in(0,1)$, define the empirical threshold discrepancy
\[
 \rho_r^Z(h,g)
 =\frac1{2q}\sum_{j=1}^{2q}\ind{|h(Z_j)-g(Z_j)|>r}.
\]
For $v\in(0,1)$, let $N_r^Z(v)$ be the smallest cardinality of a
finite set $\mathcal C\subseteq\mathcal H$ such that every
$h\in\mathcal H$ has a $g\in\mathcal C$ with
$\rho_r^Z(h,g)\le v$. If there is no finite cover, set
$N_r^Z(v)=\infty$. We only use this definition and the stated covering
bound. In particular, no triangle inequality for $\rho_r^Z$ is assumed.

The first external input turns control on sampled points into a uniform
tail guarantee. Lemma~\ref{lem:tolerant} will use its right-hand side
together with the covering estimate below.

\begin{theorem}[External tail guarantee; HKS Theorem~7]
\label{thm:hks-tail}
Fix $\eta_6,\beta\in(0,1)$, $a\in[0,1)$, an integer $q\ge1$,
and $h^*\in\mathcal H$. Let $X_1,\ldots,X_q$ be independent
with law $P$. Let $Z$ consist of these points followed by $q$ further
independent points with the same law. The probability that there is an
$h\in\mathcal H$ satisfying both
\[
 \max_{j\in[q]}|h(X_j)-h^*(X_j)|\le a\eta_6
 \quad\text{and}\quad
 P\{z:|h(z)-h^*(z)|>\eta_6\}>\beta
\]
is at most
\begin{equation}
\label{eq:app-external-tail}
 2\,\mathbb E_Z\!\left[
 N_{\eta_6(1-a)/2}^Z(\beta/8)\right]e^{-q\beta/96}.
\end{equation}
\end{theorem}

The second external input bounds the covering number in that probability
estimate by a fat-shattering dimension. Together these two inputs supply
the dimension and sample-size dependence in Lemma~\ref{lem:tolerant}.

\begin{lemma}[External covering bound; HKS Lemma~8]
\label{lem:hks-cover}
There are universal constants $C_E,c_E>0$ such that for every
$r,v\in(0,1)$ and every tuple $Z$ as above,
\begin{equation}
\label{eq:app-external-cover}
 N_r^Z(v)
 \le\left(\frac2{rv}\right)^{C_E\fat(\mathcal H,c_Erv)}.
\end{equation}
\end{lemma}

The constants $C_E$ and $c_E$ are fixed once a valid pair from
Lemma~\ref{lem:hks-cover} is chosen. They are independent of the
class, the domain, the distribution, the sample size, and all accuracy
parameters. Define
\begin{equation}
\label{eq:app-external-constants}
 c_1=192\max\{C_E,1\},\qquad
 c_2=\frac{c_E}{16},\qquad c_3=32.
\end{equation}
These are the constants used in the construction. In particular, the
constant $c_E$ from the external estimate is distinct from our final
scale constant $c'=c_2/(128\max\{1,c_2\})$.

\subsection{Deriving the Tolerant Sampling Guarantee}
\label{app:tolerant}

The next lemma combines precisely the two external inputs above. Its
role is to produce the good tuples in Lemma~\ref{lem:tuple}. This is
the tolerant sampling guarantee of HKS Theorem~9, with constants fixed
as in~\eqref{eq:app-external-constants} and with the range of $\delta$
needed here.

\begin{lemma}[Tolerant sampling on a finite domain]
\label{lem:tolerant}
Under the notation of Appendix~\ref{app:external}, fix
$\eta_6,\beta\in(0,1)$, $a\in[0,1)$,
$\delta\in(0,1/2]$, and $h^*\in\mathcal H$. Put
\[
 s=\eta_6\beta(1-a),\qquad
 d_0=\fat(\mathcal H,c_2s).
\]
If $d_0<\infty$ and
\begin{equation}
\label{eq:app-tolerant-q}
 q\ge
 \left\lceil\frac{c_1}{\beta}
 \left(d_0\ln\frac{c_3}{s}+\ln\frac1\delta\right)\right\rceil,
\end{equation}
then, with probability at least $1-\delta$ over $P^q$, the following
implication holds simultaneously for all $h\in\mathcal H$:
\begin{equation}
\label{eq:app-tolerant-implication}
 \max_{j\in[q]}|h(X_j)-h^*(X_j)|\le a\eta_6
 \quad\Longrightarrow\quad
 P\{z:|h(z)-h^*(z)|>\eta_6\}\le\beta.
\end{equation}
\end{lemma}

\begin{proof}
Set $r=\eta_6(1-a)/2$ and $v=\beta/8$. These numbers belong
to $(0,1)$, and $rv=s/16$. By Lemma~\ref{lem:hks-cover}, for
every realization of $Z$,
\[
 N_r^Z(v)
 \le\left(\frac{32}{s}\right)^%
 {C_E\fat(\mathcal H,c_Es/16)}
 =\left(\frac{32}{s}\right)^{C_Ed_0}.
\]
The last equality uses $c_2=c_E/16$. The bound does not depend on
$Z$, so it also bounds the expectation in
Theorem~\ref{thm:hks-tail}. That theorem bounds the probability of
failure of~\eqref{eq:app-tolerant-implication} by
\begin{equation}
\label{eq:app-tolerant-failure}
 2\exp\left(C_Ed_0\ln\frac{32}{s}-\frac{q\beta}{96}\right).
\end{equation}

For clarity, write $M=\max\{C_E,1\}$ and $L=\ln(32/s)$.
We have $0<s<1$, so $L>0$. Assumption~\eqref{eq:app-tolerant-q}
and the definition $c_1=192M$ imply
\[
 \frac{q\beta}{96}
 \ge2M\left(d_0L+\ln\frac1\delta\right)
 \ge C_Ed_0L+2\ln\frac1\delta.
\]
Because $\delta\le1/2$,
$\ln(2/\delta)\le2\ln(1/\delta)$. Consequently,
\[
 \frac{q\beta}{96}
 \ge C_Ed_0L+\ln\frac2\delta.
\]
Substituting this inequality into~\eqref{eq:app-tolerant-failure}
gives a failure probability at most
$2\exp(-\ln(2/\delta))=\delta$, as required.
\end{proof}

\subsection{Good Tuples and Deterministic Selection}
\label{app:tuple}

We now apply Lemma~\ref{lem:tolerant} to the finite restriction class.
The parameters in this subsection are those of the main construction:
\[
 C_2=\max\{1,c_2\},\qquad
 \eta=\frac\alpha{12C_2},\qquad
 c'=\frac{c_2}{128C_2},\qquad
 d=\fat(\F,c'\alpha),
\]
where $1\le d<\infty$. Recall the fixed tuple length
\begin{equation}
\label{eq:app-n}
 n=\left\lceil\frac{8c_1}{3}
 \left(d\ln\frac{32c_3}{3\eta}+\ln2\right)\right\rceil.
\end{equation}
The internal target $f^*$ has been fixed by
Proposition~\ref{prop:target}. The next lemma does not need its
near-optimality; membership in $\F$ is sufficient. It supplies the
existence part of Proposition~\ref{prop:weak}.

\begin{lemma}[Good tuples for an arbitrary index distribution]
\label{lem:tuple}
For every distribution $P$ on $[m]$, a tuple
$I=(I_1,\ldots,I_n)\in[m]^n$ with independent entries of law $P$
has probability at least $1/2$ of satisfying
\begin{equation}
\label{eq:app-good-tuple}
 \begin{split}
 \forall f\in\F:\quad
 &\max_{\ell\in[n]}|f(x_{I_\ell})-f^*(x_{I_\ell})|
       \le\eta/4\\
 &\hspace{12mm}\Longrightarrow\quad
 \sum_{i=1}^m P(i)\ind{|f(x_i)-f^*(x_i)|>\eta/2}
       \le3/8.
 \end{split}
\end{equation}
In particular, the set of such tuples is nonempty and has a unique
lexicographically least element.
\end{lemma}

\begin{proof}
Apply Lemma~\ref{lem:tolerant} to $\mathcal H=\GS$ on
$\Omega=[m]$, with target $g_{f^*}(i)=f^*(x_i)$. Choose
\[
 \eta_6=\eta/2,\qquad \beta=3/8,\qquad
 a=1/2,\qquad \delta=1/2.
\]
These values lie in the ranges required by that lemma. The scale in
its dimension term is
\begin{equation}
\label{eq:app-scale-identity}
 c_2\eta_6\beta(1-a)
 =c_2\frac\eta2\frac38\frac12
 =\frac{3c_2\eta}{32}
 =\frac{c_2\alpha}{128C_2}=c'\alpha.
\end{equation}
Write $d_S=\fat(\GS,c'\alpha)$. Lemma~\ref{lem:restriction}
gives $d_S\le d$. The integer sufficient for the restricted class is
\[
 n_0=\left\lceil\frac{8c_1}{3}
 \left(d_S\ln\frac{32c_3}{3\eta}+\ln2\right)\right\rceil.
\]
Since $\eta\le1/12$ and $c_3\ge1$,
$32c_3/(3\eta)\ge128$. Its logarithm is positive, so
$1\le n_0\le n$.

Take an independent length-$n$ tuple of indices. Its first $n_0$
entries have distribution $P^{n_0}$. By
Lemma~\ref{lem:tolerant}, with probability at least $1/2$ this
prefix has the following property: every function in $\GS$ that
is within $\eta/4$ of the target at each prefix position has
$P$-mass at most $3/8$ of errors larger than $\eta/2$.

Whenever this property holds for the prefix, it implies
\eqref{eq:app-good-tuple} for the full tuple. Indeed, a function
satisfying the tolerance condition at all $n$ entries also satisfies
it at the first $n_0$ entries. Its restricted function therefore has
the asserted tail bound. The definition of $g_f$ converts this bound
exactly into the sum in~\eqref{eq:app-good-tuple}.

Thus the full tuple is good with probability at least $1/2$.
The set of good tuples is consequently nonempty. It is a subset of
the finite totally ordered set $[m]^n$, with its lexicographic order.
Every nonempty such subset has a unique least element.
\end{proof}

The conclusion applies to every fixed distribution $P$. At each round,
the construction selects the least good tuple for its current $P_t$.
After this selection and the prescribed update, the next distribution is
fixed and the lemma applies again. This is an inductive definition of
deterministic choices. The probability in the proof establishes existence
at each use; it is not a failure probability retained by the scheme.
No union bound over rounds is needed. Deciding whether a tuple satisfies
the universal condition in~\eqref{eq:app-good-tuple} need not be
computable, which is consistent with the existence statement of the
main theorem.

\subsection{Quantized Constraints Determine the Weak Function}
\label{app:weak-selection}

Building on the accuracy implication of Lemma~\ref{lem:tuple},
Lemma~\ref{lem:weak-selection} supplies the quantized description required
by Proposition~\ref{prop:weak}. It proves the quantization estimate
explained in Section~\ref{sec:weak} and shows that the constraints
determine the selected function through the shared rule, as required
for exact recovery in Section~\ref{sec:protocol}.

\begin{lemma}[Selection from quantized target values]
\label{lem:weak-selection}
Fix a distribution $P$ on $[m]$. Let $I$ satisfy~\eqref{eq:app-good-tuple} for this $P$, and put
$V=\{x_{I_\ell}:\ell\in[n]\}$. Set
$J=\lceil16/\eta\rceil$. For each $x\in V$, let $j_x$ be the
integer in $\{0,\ldots,J\}$ minimizing $|j_x/J-f^*(x)|$, breaking
ties toward the smaller integer. Put $q_x=j_x/J$ and define
\[
 H(V,q)=\{f\in\F:|f(x)-q_x|\le\eta/16
                    \text{ for every }x\in V\}.
\]
Then $H(V,q)$ is nonempty. The function $h=\Sel(H(V,q))$ satisfies
\begin{equation}
\label{eq:app-weak-error}
 P\{i\in[m]:|h(x_i)-f^*(x_i)|>\eta/2\}\le3/8.
\end{equation}
Moreover, $h$ is determined by $V$, $(j_x)_{x\in V}$, and the fixed
components of the scheme. Its definition does not depend on the order
or multiplicity of the entries of $I$. The grid indices occupy
$|V|b_q\le nb_q$ bits, where $b_q=\lceil\log_2(J+1)\rceil$.
\end{lemma}

\begin{proof}
The grid contains both endpoints of $[0,1]$ and has spacing $1/J$.
Every value in $[0,1]$ is within $1/(2J)$ of a nearest grid point.
Since $J\ge16/\eta$, for each $x\in V$ we have
\begin{equation}
\label{eq:app-grid-error}
 |q_x-f^*(x)|\le\frac1{2J}\le\frac\eta{32}.
\end{equation}
In particular, $f^*$ satisfies every defining constraint of $H(V,q)$.
The set is therefore nonempty, and the fixed rule selects an element
$h\in H(V,q)$.

For every $x\in V$, the triangle inequality and
\eqref{eq:app-grid-error} give
\[
 |h(x)-f^*(x)|
 \le |h(x)-q_x|+|q_x-f^*(x)|
 \le\frac\eta{16}+\frac\eta{32}
 =\frac{3\eta}{32}<\frac\eta4.
\]
Each input $x_{I_\ell}$ belongs to $V$, so the tolerance condition
in~\eqref{eq:app-good-tuple} holds for $h$. As $h\in\F$, the
conclusion of that implication is exactly~\eqref{eq:app-weak-error}.

It remains to establish the description claim. Repeating an input
in $I$ merely repeats its constraint $|f(x)-q_x|\le\eta/16$.
Deleting a repeated constraint does not change the functions satisfying
all constraints. Reordering them also leaves this set unchanged.
Thus the set $H(V,q)$ is determined by the distinct inputs $V$ and
their grid indices. Once this set is fixed, $\Sel$ chooses the same
element of $\F$, as an entire function on $\X$.

Finally, there are $J+1$ possible grid indices. The integer
$b_q=\lceil\log_2(J+1)\rceil$ satisfies $2^{b_q}\ge J+1$, so
each index has a fixed-length binary representation using $b_q$ bits.
At most $n$ distinct inputs occur in a tuple of length $n$.
The indices consequently use $|V|b_q\le nb_q$ bits. This count
concerns the grid indices; the later protocol counts its other fields
separately.
\end{proof}

\subsection{Proof of Proposition~\ref{prop:weak}}
\label{app:weak-proof}

This proposition supports Section~\ref{sec:weak}. It states both
the weak-error guarantee needed by the distribution update and the
finite description needed by the decoder. Its existence part follows
from Lemma~\ref{lem:tuple}, which uses
Lemmas~\ref{lem:restriction} and~\ref{lem:tolerant}. Its selection
and description parts follow from Lemma~\ref{lem:weak-selection}.

\begin{proof}
Fix $S$, its selected target $f^*$, and a distribution $P$ on $[m]$.
Choose the lexicographically least tuple supplied by
Lemma~\ref{lem:tuple}. This determines its support $V$. Quantize
$f^*(x)$ on this support with the fixed rule of
Lemma~\ref{lem:weak-selection}, and select $h=\Sel(H(V,q))$.
That lemma proves the constraint set is nonempty and gives
\[
 P\{i:|h(x_i)-f^*(x_i)|>\eta/2\}\le\frac38.
\]
It also proves that the support and its grid indices determine the
whole function $h$. The support contains at most $n$ inputs, and the
indices require at most $nb_q$ bits.

Each retained input can be represented by an original labeled example:
for $x\in V$, choose the position
$i_x=\min\{i\in[m]:x_i=x\}$ and retain $(x_{i_x},y_{i_x})$.
These positions exist because $V$ consists of sample inputs.
Different retained inputs give different positions, so there are at
most $n$ retained examples. Their original labels are left unchanged.
The value $q_x$ is represented in the auxiliary bits, rather than
substituted for the original label. The grid index is a fixed function
of $x$ and $f^*(x)$; if the same input occurs in another round, its
index is therefore the same.

All selections use fixed rules. Consequently this construction defines
a deterministic weak function for every current distribution $P$.
The sampling argument is used only to show that the selected good
tuple exists.
\end{proof}

\section{Detailed Proofs for Fixed-Step Boosting and Median Aggregation}
\label{app:boost}

This appendix proves Proposition~\ref{prop:boost} for the updates and
aggregation in Section~\ref{sec:boost}.
Lemma~\ref{lem:fixed-update} establishes valid updates and bounds their
normalizing factors.
Lemma~\ref{lem:median} identifies the upper median and gives its interval
guarantee.
Lemma~\ref{lem:potential} controls inaccurate median predictions through
an exponential potential.
Appendix~\ref{app:boost-proof} combines these results with the weak-function
guarantee of Proposition~\ref{prop:weak}.

Exponential reweighting and median aggregation follow the approach of
\citet{kegl2003robust}, used for regression compression by
\citet{hanneke2019sample}. We give a self-contained analysis of the fixed
coefficient and equal-weight rule used in the present construction.
All accuracy statements in this appendix concern the predictor $g$
formed inside the compressor. Proposition~\ref{prop:protocol} subsequently
establishes that the decoder reconstructs this same function.

\subsection{A Finite Update with a Uniform Contraction}
\label{app:fixed-update}

The first lemma supplies the local invariant for the loop in
Section~\ref{sec:boost}. Its probability and support assertions justify
the next invocation of the weak-function stage. Its bound on the
normalizing factor will provide the contraction needed for the final
accuracy estimate. Throughout this appendix, the constants are
\begin{equation}\label{eq:app-fixed-constants}
a_0=\frac12\ln\frac53,
\qquad \lambda_0=\frac{\sqrt{15}}4.
\end{equation}
In particular, $a_0>0$ and $0<\lambda_0<1$.

\begin{lemma}[Fixed-step update and contraction]
\label{lem:fixed-update}
Let $m\ge1$ be an integer, let $P$ be a probability distribution on $[m]$ with
$P(i)>0$ for every $i$, and let $E\subseteq[m]$. Put
\[
w=P(E),\qquad
\theta_i=1-2\cdot\ind{i\in E},\qquad
Z=\sum_{i=1}^m P(i)e^{-a_0\theta_i}.
\]
Then $0<Z<\infty$, and
\begin{equation}\label{eq:app-fixed-update}
P^+(i)=\frac{P(i)e^{-a_0\theta_i}}{Z}
\end{equation}
defines a probability distribution with positive mass at every index.
Moreover,
\begin{equation}\label{eq:app-fixed-contraction}
w\le\frac38\quad\Longrightarrow\quad Z\le\lambda_0.
\end{equation}
If $w=0$, then $E=\varnothing$, $Z=e^{-a_0}$, and $P^+=P$.
\end{lemma}

\begin{proof}
For each $i$, the number $\theta_i$ is either $1$ or $-1$.
Since $a_0$ is a finite real number, $e^{-a_0\theta_i}$ is finite and
strictly positive. Each summand in $Z$ is therefore finite and strictly
positive. There are $m$ summands, with $m\ge1$, so $0<Z<\infty$.
Consequently every value in \eqref{eq:app-fixed-update} is well-defined
and positive. Summing that equation gives
\[
\sum_{i=1}^mP^+(i)
=\frac1Z\sum_{i=1}^mP(i)e^{-a_0\theta_i}=1.
\]
Thus $P^+$ is a probability distribution with the required support.

We next group the terms defining $Z$ according to membership in $E$.
The total mass of $E$ is $w$, and its complement has mass $1-w$.
The sign is $-1$ on $E$ and $1$ on its complement. Hence
\begin{equation}\label{eq:app-normalizer-by-error}
Z=(1-w)e^{-a_0}+we^{a_0}
 =e^{-a_0}+w\bigl(e^{a_0}-e^{-a_0}\bigr).
\end{equation}
The coefficient of $w$ is positive because $a_0>0$. If $w\le3/8$,
the right-hand side is therefore at most its value at $w=3/8$.
Using
\[
e^{a_0}=\sqrt{\frac53},\qquad
e^{-a_0}=\sqrt{\frac35},
\]
we obtain
\begin{align*}
Z
&\le\frac58\sqrt{\frac35}+\frac38\sqrt{\frac53}\\
&=\frac{\sqrt{15}}8+\frac{\sqrt{15}}8
 =\frac{\sqrt{15}}4=\lambda_0.
\end{align*}
The strict inequality $\lambda_0<1$ follows from $15<16$.
This proves \eqref{eq:app-fixed-contraction}.

Finally, suppose $w=0$. If $E$ contained an index $i_0$, then
$w=\sum_{i\in E}P(i)\ge P(i_0)>0$, a contradiction.
Thus $E=\varnothing$ and every sign equals $1$.
It follows that
\[
Z=e^{-a_0}\sum_iP(i)=e^{-a_0},\qquad
P^+(i)=\frac{P(i)e^{-a_0}}{e^{-a_0}}=P(i).
\]
The update is therefore defined in the zero-error case as well.
\end{proof}

The uniform distribution has positive mass at every index. Applying the
first part of Lemma~\ref{lem:fixed-update} at successive rounds preserves
this property. The contraction part only requires the weak-error bound
at the current round, so it can be applied after that round's function
has been selected.

\subsection{The Equal-Weight Upper Median}
\label{app:median}

The next lemma supports the aggregation rule in Section~\ref{sec:boost}.
It proves that the rule defines a prediction at every input and identifies
the sufficient local condition for that prediction to be accurate.
It will also justify the same rule when it is applied by the decoder.

\begin{lemma}[Upper median and a strict majority in an interval]
\label{lem:median}
Let $T\ge1$ be an integer and let $z_1,\ldots,z_T\in\mathbb R$.
Repeated values are allowed. Define
\begin{equation}\label{eq:app-equal-median-rule}
M=\min\left\{z_j:j\in[T],\
             \sum_{t=1}^T\ind{z_t>z_j}<\frac T2\right\}.
\end{equation}
The candidate set is nonempty. If
$z_{(1)}\le\cdots\le z_{(T)}$ is the ordered list, with repetitions
retained, then
\begin{equation}\label{eq:app-median-order}
M=z_{(k)},\qquad k=\lfloor T/2\rfloor+1.
\end{equation}
In particular, $M$ is one of the given values and
\begin{equation}\label{eq:app-equal-median-halves}
\#\{t:z_t\le M\}>T/2,
\qquad
\#\{t:z_t\ge M\}\ge T/2.
\end{equation}
For any real numbers $\ell\le u$,
\begin{equation}\label{eq:app-majority-interval}
\#\{t:z_t\in[\ell,u]\}>T/2
\quad\Longrightarrow\quad M\in[\ell,u].
\end{equation}
Consequently, if every $z_t$ belongs to $[0,1]$, then $M\in[0,1]$.
\end{lemma}

\begin{proof}
Let $z_{\max}=\max_{t\in[T]}z_t$. No prediction is strictly greater
than $z_{\max}$, so the sum in \eqref{eq:app-equal-median-rule} is zero
at that candidate. Since $T>0$, this candidate satisfies the strict
inequality. The candidate set is finite and nonempty, and its minimum
exists.

To identify the minimum, put $k=\lfloor T/2\rfloor+1$.
There are at most $T-k$ indices whose values are strictly larger than
$z_{(k)}$. For both parities of $T$ one has $T-k<T/2$.
Indeed, if $T=2r$, then $k=r+1$ and $T-k=r-1<r$.
If $T=2r+1$, then $k=r+1$ and $T-k=r<r+1/2$.
Thus $z_{(k)}$ is a candidate in \eqref{eq:app-equal-median-rule}.

Now consider any given value $z_j<z_{(k)}$. All the entries in ordered
positions $k,k+1,\ldots,T$ are strictly larger than $z_j$.
They give at least
\[
T-k+1=\lceil T/2\rceil\ge T/2
\]
indices counted in the upper tail of $z_j$.
Therefore $z_j$ is not a candidate. A candidate equal to $z_{(k)}$
exists, and no smaller given value qualifies. This proves
\eqref{eq:app-median-order}. The argument counts ordered entries,
so it remains valid when some of the values coincide.

At least the first $k$ ordered entries are at most $M=z_{(k)}$.
Since $k>T/2$, this proves the first inequality in
\eqref{eq:app-equal-median-halves}. At least the entries in positions
$k,\ldots,T$ are at least $M$, and there are
$T-k+1=\lceil T/2\rceil$ of them. This proves the second inequality.

For the interval assertion, let
\[
A=\{t\in[T]:z_t\in[\ell,u]\}
\]
and suppose $|A|>T/2$. Then $|[T]\setminus A|<T/2$.
If $M<\ell$, every index with $z_t\le M$ lies outside $A$.
It would follow that
\[
\#\{t:z_t\le M\}\le|[T]\setminus A|<T/2,
\]
contrary to the first inequality in \eqref{eq:app-equal-median-halves}.
If $M>u$, every index with $z_t\ge M$ lies outside $A$.
This would give
\[
\#\{t:z_t\ge M\}\le|[T]\setminus A|<T/2,
\]
contrary to the second inequality. Thus $\ell\le M\le u$,
which proves \eqref{eq:app-majority-interval}.
The final range assertion also follows directly from the fact that
$M$ is one of the given values.
\end{proof}

For odd $T$ the selected order statistic is the middle entry. For even
$T$ it is the upper of the two middle entries. In particular, when
$T=2r$, with $r$ values equal to $0$ and $r$ equal to $1$, the rule
returns $1$. Exactly half the values then belong to the closed interval
$[0,0]$, but the median does not. Thus the strict-majority hypothesis
in \eqref{eq:app-majority-interval} cannot be replaced by a nonstrict
one. This boundary case will be included in the failure event of the
next lemma.

\subsection{An Exponential Potential for Inaccurate Median Predictions}
\label{app:potential}

The third lemma connects the preceding two ingredients. It turns the
failure of the majority condition at a sample position into a bound by
an exponential potential, then expresses the average potential using
the normalizing factors of the updates. The lemma does not impose a
weak-error bound. That bound will enter when we prove
Proposition~\ref{prop:boost}. The potential identity highlighted in
Section~\ref{sec:boost} is the specialization of this lemma to the
constructed rounds.

\begin{lemma}[Median error bounded by the update potential]
\label{lem:potential}
Let $m,T\ge1$ be integers, let $x_1,\ldots,x_m\in\X$, and fix
$f^*:\X\to[0,1]$ and $\eta>0$.
Let $h_1,\ldots,h_T:\X\to[0,1]$ be any functions.
For each $i\in[m]$ and $t\in[T]$, put
\[
y_i^*=f^*(x_i),\qquad
\theta_i^{(t)}=
\begin{cases}
1,&|h_t(x_i)-y_i^*|\le\eta/2,\\
-1,&|h_t(x_i)-y_i^*|>\eta/2.
\end{cases}
\]
Starting from $P_1(i)=1/m$, define
\begin{equation}\label{eq:app-potential-update}
Z_t=\sum_{i=1}^mP_t(i)e^{-a_0\theta_i^{(t)}},
\qquad
P_{t+1}(i)=\frac{P_t(i)e^{-a_0\theta_i^{(t)}}}{Z_t}
\quad(t\in[T]).
\end{equation}
Let $g(x)$ be the upper median in \eqref{eq:app-equal-median-rule}
for the list $(h_t(x))_{t=1}^T$, and define
\[
B=\{i\in[m]:|g(x_i)-y_i^*|>\eta/2\},
\qquad D_i=\sum_{t=1}^T\theta_i^{(t)}.
\]
Then all updates are defined and
\begin{equation}\label{eq:app-median-potential}
\frac{|B|}{m}
\le\frac1m\sum_{i=1}^m e^{-a_0D_i}
=\prod_{t=1}^T Z_t.
\end{equation}
\end{lemma}

\begin{proof}
The initial distribution has full support. The first part of
Lemma~\ref{lem:fixed-update} therefore applies inductively to every
update in \eqref{eq:app-potential-update}. Every $Z_t$ is finite and
positive, and every $P_t$ is a probability distribution with full support.
Lemma~\ref{lem:median} also shows that $g(x)$ is defined at every input.

\stagelead{Step 1: a median error implies a nonpositive signed sum.}
Fix $i\in[m]$ and let
\[
G_i=\{t\in[T]:|h_t(x_i)-y_i^*|\le\eta/2\}.
\]
Here $G_i$ is a set of rounds for one fixed position. If $|G_i|>T/2$,
strictly more than half of the predictions lie in the closed interval
$[y_i^*-\eta/2,y_i^*+\eta/2]$.
Applying \eqref{eq:app-majority-interval} to that interval gives
$|g(x_i)-y_i^*|\le\eta/2$.
Taking the contrapositive, we obtain
\[
i\in B\quad\Longrightarrow\quad |G_i|\le T/2.
\]
The positive signs are precisely the rounds in $G_i$. Hence
\[
D_i=|G_i|-(T-|G_i|)=2|G_i|-T,
\]
and therefore
\begin{equation}\label{eq:app-median-bad-sum}
i\in B\quad\Longrightarrow\quad D_i\le0.
\end{equation}
The nonstrict inequality includes the case of an even number of rounds
with exactly half the predictions accurate. Values at distance exactly
$\eta/2$ are included in $G_i$, consistently with the strict inequality
in the definition of $B$.

Since $a_0>0$, a nonpositive value of $D_i$ gives
$e^{-a_0D_i}\ge1$. If $D_i>0$, the indicator of $D_i\le0$ is zero
and the exponential remains positive. Consequently, for every index,
\begin{equation}\label{eq:app-pointwise-potential}
\ind{i\in B}\le\ind{D_i\le0}\le e^{-a_0D_i}.
\end{equation}
Averaging this pointwise inequality proves the first inequality in
\eqref{eq:app-median-potential}.

\stagelead{Step 2: the updates compute the average potential.}
For every $i\in[m]$ and $t\in[T]$, we claim that
\begin{equation}\label{eq:app-unrolled-update}
P_{t+1}(i)
=\frac1m\,
\frac{\exp\!\left(-a_0\sum_{s=1}^t\theta_i^{(s)}\right)}
     {\prod_{s=1}^t Z_s}.
\end{equation}
For $t=1$, this is \eqref{eq:app-potential-update} with
$P_1(i)=1/m$. Suppose the formula holds after some round $t<T$.
The next update multiplies its numerator by
$\exp(-a_0\theta_i^{(t+1)})$ and its denominator by $Z_{t+1}$.
The exponents add, giving the formula with $t+1$ in place of $t$.
Induction proves \eqref{eq:app-unrolled-update} through round $T$.

At $t=T$, the exponent sum is $D_i$. Summing
\eqref{eq:app-unrolled-update} over $i$ and using
$\sum_iP_{T+1}(i)=1$ yields
\[
1=\frac{1}{\prod_{t=1}^T Z_t}
       \frac1m\sum_{i=1}^m e^{-a_0D_i}.
\]
Every factor $Z_t$ is positive. Multiplication by their product gives
the equality in \eqref{eq:app-median-potential}, completing the proof.
\end{proof}

The identity holds for functions chosen from the entire preceding
history of the updates. In particular, the functions may depend on
the fixed sample and the internal target. The argument is deterministic
and assumes no independence between rounds. It counts sample positions
throughout. Repeated inputs remain separate positions in the sums and
do not affect any step of the proof.

\subsection{Proof of Proposition~\ref{prop:boost}}
\label{app:boost-proof}

Proposition~\ref{prop:boost} supplies the guarantee for
Section~\ref{sec:boost}: the prescribed loop produces a well-defined
equal-weight median that approximates the fixed target on most sample
positions. Proposition~\ref{prop:weak} provides each weak function.
Lemma~\ref{lem:fixed-update} justifies the subsequent update and bounds
its normalizer. Lemmas~\ref{lem:median} and~\ref{lem:potential} give
the range and accuracy of the final aggregate.

\begin{proof}
Fix $\F$, $\alpha$, and a sample
$S=((x_i,y_i))_{i=1}^m$ in the positive-dimension branch of
Section~\ref{sec:construction}. Thus $0<\alpha\le1$, $1\le d<\infty$,
and $m\ge1$. Keep the selected $f^*$ from \eqref{eq:target} fixed
throughout the loop. Its existence was established in
Proposition~\ref{prop:target}.
The parameters satisfy $\eta>0$ and
\[
0<\eps=\frac{\alpha^2}{64}\le\frac1{64}<1.
\]
Since $0<\lambda_0<1$, the denominator $-\ln\lambda_0$ in
\eqref{eq:round-count} is positive. Its numerator $\ln(1/\eps)$
is also positive. Thus $T$ is a finite positive integer.

\stagelead{All prescribed rounds are defined.}
Initially $P_1(i)=1/m>0$ for every $i$, and its masses sum to one.
Suppose a probability distribution $P_t$ with full support has been
constructed for some $t\in[T]$.
Proposition~\ref{prop:weak} applies to this distribution and the fixed
$S,f^*$. It provides the selected function $h_t\in\F$ and its support
description, with
\[
e_t=P_t\{i:|h_t(x_i)-f^*(x_i)|>\eta/2\}\le3/8.
\]
The set inside this probability is precisely $E_t$ in
Section~\ref{sec:boost}. Its signs agree with those in
Lemma~\ref{lem:fixed-update}, applied with $P=P_t$, $E=E_t$, and $w=e_t$. The lemma proves that the next update
is defined, that $P_{t+1}$ is a probability distribution with full
support, and that $0<Z_t\le\lambda_0$.
This completes the induction step. All $T$ functions and all
distributions $P_1,\ldots,P_{T+1}$ therefore exist, with the stated
support and normalizer bounds.

If any $e_t$ is zero, the same lemma gives $P_{t+1}=P_t$.
The next round still has a valid input distribution, so the induction
continues through the prescribed $T$ rounds in this case as well.
The weak-function choices are fixed by their selection rules, and the
update formula determines each next distribution. Thus the preceding
induction applies to the deterministic procedure specified in the body.

\stagelead{The aggregate meets the required accuracy condition.}
For every $x\in\X$, all values $h_t(x)$ belong to $[0,1]$, since
$h_t\in\F\subseteq[0,1]^{\X}$.
Lemma~\ref{lem:median} proves that \eqref{eq:equal-median} defines
$g(x)\in[0,1]$. It also verifies the upper-median convention used in
that equation for both parities of $T$ and for repeated predictions.

The signs and updates of the construction agree exactly with those
in Lemma~\ref{lem:potential}. Apply that lemma to the resulting
$h_1,\ldots,h_T$ and define
\[
B=\{i\in[m]:|g(x_i)-f^*(x_i)|>\eta/2\}.
\]
The normalizer bounds already established give
\begin{equation}\label{eq:app-final-median-error}
\frac{|B|}{m}\le\prod_{t=1}^T Z_t\le\lambda_0^T.
\end{equation}
By the definition of $T$,
\[
T(-\ln\lambda_0)\ge\ln(1/\eps).
\]
Multiplying by $-1$ and exponentiating yields
\[
\lambda_0^T=\exp(T\ln\lambda_0)
\le\exp(-\ln(1/\eps))=\eps.
\]
Set $G=[m]\setminus B$. The definition of $B$ and
\eqref{eq:app-final-median-error} now imply
\[
|[m]\setminus G|\le\eps m,
\qquad |g(x_i)-f^*(x_i)|\le\eta/2\quad(i\in G).
\]
These are precisely the two assertions about $G$ in the proposition.

\stagelead{The round budget is uniform in the sample length.}
Let $L=\ln(2/\alpha)$, so $L\ge\ln2>0$.
Using $\eps=\alpha^2/64$, we have
\[
\ln(1/\eps)
=2\ln(1/\alpha)+\ln64
=2L+4\ln2\le6L.
\]
The ceiling inequality $\lceil u\rceil\le u+1$ gives
\begin{align*}
T
&\le\frac{6L}{-\ln\lambda_0}+1\\
&\le\left(\frac6{-\ln\lambda_0}+\frac1{\ln2}\right)L.
\end{align*}
The coefficient in parentheses is a finite universal constant because
$\lambda_0=\sqrt{15}/4$ is fixed. This proves
$T=O(\ln(2/\alpha))$. Neither the exact definition of $T$ nor this
bound involves $m$.
\end{proof}

The function whose accuracy has just been proved is the compressor's
equal-weight median $g$. The reconstruction argument in
Proposition~\ref{prop:protocol} recovers the same weak functions and
applies the same median rule, giving $\hat f=g$ on all of $\X$.
Proposition~\ref{prop:target} then transfers the accuracy on $G$ to
the empirical loss against the original observed labels.

\section{The Code, Its Reconstruction, and the Final Bounds}\label{app:protocol}
Section~\ref{sec:protocol} records a short description of the ensemble constructed in Section~\ref{sec:boost}. This appendix defines its decoder on every code and proves exact reconstruction. It then counts the observations and bits used by that description. Finally, it combines these facts with the approximation guarantee to prove Theorem~\ref{thm:main}. The arguments below concern the fixed-step construction, which always produces exactly $T$ weak functions.

\subsection{A Complete Definition of the Decoder}\label{app:decode}
Fix a nonempty class $\F$, an accuracy $0<\alpha\le1$, and the orders specified in Section~\ref{sec:setting}. Assume that $d=\fat(\F,c'\alpha)<\infty$. These data are shared by the compressor and decoder. In the branch $d\ge1$, the integers $n,T,J,b_q$ and the tolerance $\eta$ are determined by these fixed data. In particular, the decoder knows $T,J,b_q$ without knowing the original sample or its length.

A code is a pair $(S',B)$, where $S'$ is any finite sequence of elements of $\X\times[0,1]$ and $B$ is any finite binary string. This domain includes empty sequences, repeated inputs, and strings that no compressor output will use. We define the reconstruction map on this entire domain.

If $d=0$, the decoder returns $f_0=\Sel(\F)$ on every code. This rule is applied before any of the checks below. For the remainder of this subsection, suppose $d\ge1$. Write $U=|S'|$, and let $0_{\X}$ denote the function that is identically zero on $\X$.

\begin{enumerate}
\item \emph{Check the retained inputs and the string length.}
Return $0_{\X}$ if $U=0$, if two observations in $S'$ have the same input, or if
\[
|B|\ne U(b_q+T).
\]
Otherwise sort the input coordinates in the fixed order on $\X$, obtaining
$\xi_1\prec\cdots\prec\xi_U$. This operation does not use the observed labels in $S'$.

\item \emph{Read the fields associated with each input.}
Partition $B$ into $U$ consecutive blocks of length $b_q+T$. For each $r\in[U]$, interpret the first $b_q$ bits of its block as an unsigned integer $j_{\xi_r}$, with most significant bit first. Read the remaining $T$ bits as
$c_{r,1},\ldots,c_{r,T}$. Return $0_{\X}$ if any $j_{\xi_r}>J$. Otherwise put
\[
q_{\xi_r}=j_{\xi_r}/J\qquad(r\in[U]).
\]
The integer $J$ is positive. Thus all these grid values are well-defined and lie in $[0,1]$.

\item \emph{Recover the constraints for each round.}
For each $t\in[T]$, form
\begin{equation}\label{eq:decoder-constraints}
\begin{split}
V_t^{\mathrm{rec}}&=\{\xi_r:r\in[U],\ c_{r,t}=1\},\\
H_t^{\mathrm{rec}}&=\{f\in\F:|f(x)-q_x|\le\eta/16
                      \text{ for every }x\in V_t^{\mathrm{rec}}\}.
\end{split}
\end{equation}
Return $0_{\X}$ if any $V_t^{\mathrm{rec}}$ or $H_t^{\mathrm{rec}}$ is empty. Otherwise set
\[
h_t^{\mathrm{rec}}=\Sel(H_t^{\mathrm{rec}})\qquad(t\in[T]).
\]
Each selection is an element of the shared class $\F$.

\item \emph{Apply the same median rule to all recovered rounds.}
For every $x\in\X$, output
\begin{equation}\label{eq:decoder-median}
\hat f(x)=\min\left\{h_j^{\mathrm{rec}}(x):j\in[T],\
\sum_{t=1}^T\ind{h_t^{\mathrm{rec}}(x)>h_j^{\mathrm{rec}}(x)}<T/2\right\}.
\end{equation}
There is one entry for each round. Equal functions obtained in different rounds remain separate entries in this list.
\end{enumerate}

These rules define a total deterministic map. Indeed, the checks either prescribe $0_{\X}$ or lead to nonempty constraint sets, which have uniquely specified selected elements. The integer $T$ is positive because $0<\eps<1$ and $0<\lambda_0<1$. At a fixed $x$, a largest value among the $T$ recovered predictions is a candidate in \eqref{eq:decoder-median}: no prediction lies strictly above it, and $0<T/2$. The candidate set is finite and nonempty, so its minimum exists. Lemma~\ref{lem:median} identifies this rule with the upper median, including ties and repeated values. It also shows that the result lies between the smallest and largest recovered predictions. Hence every output of the decoder takes values in $[0,1]$.

This is a definition of a mathematical map; deciding whether a constraint set is empty need not be computable. The decoder also does not attempt to determine whether $S'$ is a subsequence of an unknown original sample. That requirement concerns outputs of the compressor and is proved next. The observed labels in $S'$ remain part of the stored original observations, although they are not used in forming the quantized constraints.

\subsection{Proof of Proposition~\ref{prop:protocol}}
\label{app:protocol-proof}
This proposition justifies the reconstruction stage in Section~\ref{sec:protocol}. The main supporting result is Lemma~\ref{lem:weak-selection}: a round's input set and grid indices determine its selected function. Proposition~\ref{prop:target} defines the internal target. Together, Proposition~\ref{prop:weak} and Lemma~\ref{lem:fixed-update} ensure that the construction supplies all $T$ round descriptions. Lemma~\ref{lem:median} makes the final aggregation well-defined. The proof checks each field before establishing the selection identity explained in Section~\ref{sec:protocol}.

\begin{proof}
Fix an arbitrary finite sample $S=((x_i,y_i))_{i=1}^m$, where $m\ge1$. In the case $d=0$, the compressor returns the pair of empty sequences, and the decoder returns $f_0$ on that code. This is exactly the designated output of the zero-dimensional branch. We therefore assume $d\ge1$ and verify every claim for the positive-dimensional construction.

\stagelead{The retained sequence consists of original observations.}
Proposition~\ref{prop:weak} supplies a good tuple of length $n$ at each round. The integer $n$ is at least one, so each support $V_t$ is nonempty and consists of inputs that occur in $S$. There are exactly $T$ rounds. Thus
\[
V=\bigcup_{t=1}^T V_t
\]
is a finite, nonempty subset of $\{x_i:i\in[m]\}$. For every $x\in V$, the set
$\{i\in[m]:x_i=x\}$ has a least element $i_x$. If $x\ne x'$, then $i_x\ne i_{x'}$, because a single sample position has only one input coordinate.

Arrange these distinct indices as $i_{(1)}<\cdots<i_{(U)}$, where $U=|V|$. The sequence
\[
S'=((x_{i_{(r)}},y_{i_{(r)}}))_{r=1}^U
\]
is therefore a subsequence of $S$ in the sense of Definition~\ref{def:compression}. Each selected input occurs once in $S'$, and every selected label is its original observed label. In particular, a stored label is not replaced by a value of $f^*$ or by a grid value.

The compressor uses the fixed order on $\X$ to list the same set as
$\xi_1\prec\cdots\prec\xi_U$. The decoder obtains precisely this list by sorting the inputs of $S'$. The order used for the bit fields therefore agrees at both ends, even though the stored observations themselves are listed in their original sample order. This establishes the association of the $r$-th block of $B$ with the input $\xi_r$. No source index needs to be transmitted to recover this association.

\stagelead{The grid and membership fields are recovered exactly.}
For each $x\in V$, the compressor's index $j_x$ belongs to $\{0,\ldots,J\}$. Since
\[
J+1\le 2^{b_q},
\]
its unsigned representation in $b_q$ bits is defined and unique. The stored membership values are individual bits, with
\begin{equation}\label{eq:app-incidence}
c_{r,t}=\ind{\xi_r\in V_t}\qquad(r\in[U],\ t\in[T]).
\end{equation}
There are exactly $U$ blocks. Each contains one grid field followed by $T$ membership bits, so the constructed string has length $U(b_q+T)$.

The decoder obtains $U$ from $S'$ and already knows both field widths. It consequently partitions the string into exactly the blocks written by the compressor. Its recovered integers and bits are the original $j_{\xi_r}$ and $c_{r,t}$. In particular, every grid-index check succeeds, and the decoded value $q_x=j_x/J$ is exactly the compressor's value. The single grid index assigned to an input is valid in every round using that input, because the target $f^*$, the grid, and the tie rule were fixed throughout the construction.

\stagelead{Every round support is recovered in both directions.}
Fix $t\in[T]$. If $x\in V_t^{\mathrm{rec}}$, then $x=\xi_r$ for some $r$ with $c_{r,t}=1$. By \eqref{eq:app-incidence}, $x\in V_t$. Conversely, if $x\in V_t$, then $x\in V$ and hence $x=\xi_r$ for a unique $r\in[U]$. Its membership bit is one, so $x\in V_t^{\mathrm{rec}}$. This proves
\begin{equation}\label{eq:app-support-recovery}
V_t^{\mathrm{rec}}=V_t\qquad(t\in[T]).
\end{equation}
In particular, all the reconstructed supports are nonempty. This recovery concerns sets of distinct inputs. The order and multiplicities in a round's original tuple are unnecessary for its constraint set, as established in Lemma~\ref{lem:weak-selection}.

\stagelead{The selected functions agree on their entire domain.}
Let $H_t=H(V_t,q)$ be the compressor's constraint set for round $t$. The equality of supports in \eqref{eq:app-support-recovery} and the exact recovery of every grid value give
\[
H_t^{\mathrm{rec}}
=\{f\in\F:|f(x)-q_x|\le\eta/16\text{ for every }x\in V_t\}
=H_t.
\]
These sets are nonempty. More explicitly, nearest-grid quantization and the choice $J=\lceil16/\eta\rceil$ give
\[
|f^*(x)-q_x|\le\frac1{2J}\le\frac\eta{32}<\frac\eta{16}
\qquad(x\in V).
\]
Thus $f^*\in H_t$ for every $t$, consistently with Lemma~\ref{lem:weak-selection}. The decoder does not need this witness in order to apply its rule. Its role here is to certify that a compressor-produced code passes the nonemptiness checks.

Both maps use the same selector on the same class. Therefore
\begin{equation}\label{eq:app-function-recovery}
h_t^{\mathrm{rec}}=\Sel(H_t^{\mathrm{rec}})
=\Sel(H_t)=h_t\qquad(t\in[T]).
\end{equation}
This equality identifies elements of $\F$. It follows that their values agree at every input in $\X$, including inputs absent from $S$.

\stagelead{The final median is the same function.}
All $T$ membership columns are retained. If two rounds have the same support, their two columns remain separate. If their selected functions are equal, their two occurrences still contribute separately to both median expressions. Hence, for every $x\in\X$ and $j\in[T]$, \eqref{eq:app-function-recovery} gives
\[
\sum_{t=1}^T\ind{h_t^{\mathrm{rec}}(x)>h_j^{\mathrm{rec}}(x)}
=\sum_{t=1}^T\ind{h_t(x)>h_j(x)}.
\]
The candidate values themselves agree as well. The finite candidate sets in \eqref{eq:decoder-median} and \eqref{eq:equal-median} are thus identical. They are nonempty by the median rule established in Lemma~\ref{lem:median}. Taking their minima proves
\begin{equation}\label{eq:app-exact-output}
\hat f(x)=g(x)\qquad\text{for every }x\in\X.
\end{equation}

\stagelead{Validity, determinism, and the information used.}
We have verified all checks in Appendix~\ref{app:decode}: the stored inputs are nonempty and distinct, the string has the required length, every grid index is in range, and each reconstructed support and constraint set is nonempty. A code produced by the compressor is therefore accepted. The construction of $S'$ already proves the separate original-subsequence requirement.

The compression map is deterministic. Proposition~\ref{prop:target} gives a nonempty set from which the fixed selector chooses $f^*$. At each round, the lexicographically least good tuple is unique, because the finite set of good tuples is nonempty by Proposition~\ref{prop:weak}. Grid ties are resolved by the smaller integer. The selection rule fixes $h_t$, and the fixed-step update fixes $P_{t+1}$. Lemma~\ref{lem:fixed-update} guarantees that the normalization is defined. Induction from the fixed distribution $P_1$ determines the entire sequence of $T$ rounds. The first-occurrence representatives, their original ordering, and the binary encoding are also fixed.

The decoder is the total deterministic map already defined in Appendix~\ref{app:decode}. Its instructions use only the code, the shared class and orders, and the fixed numerical parameters. In particular, they do not use the target, the unsaved observations, the distributions used in the rounds, or the original sample length. These facts complete the proposition.
\end{proof}

\subsection{Proof of Proposition~\ref{prop:size}}\label{app:size}
This proposition supplies the budget in Section~\ref{sec:close}. Its inputs are the support-size bound from Proposition~\ref{prop:weak}, the fixed round count from Proposition~\ref{prop:boost}, and the code layout verified in the proof of Proposition~\ref{prop:protocol}. We first count the exact description. We then bound each parameter uniformly for all $0<\alpha\le1$.

\begin{proof}
Suppose $d\ge1$. Each support $V_t$ contains at most $n$ inputs. Therefore
\[
U=\left|\bigcup_{t=1}^T V_t\right|
\le\sum_{t=1}^T|V_t|\le Tn.
\]
The retained sequence contains exactly one original observation for each input of the union. The string has one $b_q$-bit grid field and $T$ membership bits for each such input. It follows that
\begin{equation}\label{eq:app-full-size}
\begin{split}
|S'|&=U\le Tn,\\
|B|&=U(b_q+T)\le Tn(b_q+T),\\
|S'|+|B|&=U(1+b_q+T)\le Tn(1+b_q+T).
\end{split}
\end{equation}
This counts every transmitted component. The two components of a code are finite sequences, whose lengths are available to the decoder. Their lengths and the known field widths determine all block boundaries. There are no additional length or separator fields in this description.

Put $L=\ln(2/\alpha)$. Since $0<\alpha\le1$, we have $L\ge\ln2>0$. This lower bound allows all constant terms from rounding to be absorbed uniformly, including at the endpoint $\alpha=1$.

\stagelead{The number of inputs used in one round.}
The equality $\eta=\alpha/(12C_2)$ gives
\[
\ln\frac{32c_3}{3\eta}
=\ln\frac{128c_3C_2}{\alpha}
=L+\ln(64c_3C_2).
\]
Let $A_n=\ln(64c_3C_2)$, a positive universal constant. Applying $\lceil z\rceil\le z+1$ in the definition of $n$ yields
\[
n\le\frac{8c_1}{3}\bigl(d(L+A_n)+\ln2\bigr)+1.
\]
Here $dA_n\le (A_n/\ln2)dL$, $\ln2\le dL$, and $1\le dL/\ln2$. Hence
\begin{equation}\label{eq:app-n-bound}
n\le C_n dL,
\qquad C_n=\frac{8c_1}{3}\left(2+\frac{A_n}{\ln2}\right)
             +\frac1{\ln2}.
\end{equation}
The constant $C_n$ depends only on the universal constants fixed in Section~\ref{sec:setting}.

\stagelead{The number of rounds.}
Since $\eps=\alpha^2/64$,
\[
\ln(1/\eps)=2\ln(1/\alpha)+\ln64
=2L+4\ln2\le6L.
\]
The fixed number $\lambda_0=\sqrt{15}/4$ satisfies $0<\lambda_0<1$, so $-\ln\lambda_0$ is a positive numerical constant. The formula \eqref{eq:round-count} therefore implies
\begin{equation}\label{eq:app-round-bound}
T\le\frac{6L}{-\ln\lambda_0}+1\le C_T L,
\qquad C_T=\frac6{-\ln\lambda_0}+\frac1{\ln2}.
\end{equation}
In particular, the ceiling contributes a bounded term even when $\alpha=1$.

\stagelead{The number of grid bits per retained input.}
The two ceilings in $J$ and $b_q$ are accounted for separately. First,
\[
J+1\le\frac{16}{\eta}+2
=\frac{192C_2}{\alpha}+2
\le\frac{192C_2+2}{\alpha}.
\]
Let $A_q=192C_2+2$. Then
\[
b_q\le1+\log_2(J+1)
\le1+\log_2 A_q+\frac{\ln(1/\alpha)}{\ln2}.
\]
Using $\ln(1/\alpha)\le L$ and $L\ge\ln2$ gives
\begin{equation}\label{eq:app-grid-bound}
b_q\le C_qL,
\qquad C_q=\frac{2+\log_2 A_q}{\ln2}.
\end{equation}
This is again a universal bound, independent of the class, the sample, and its length.

\stagelead{Separate and total budgets.}
Combining \eqref{eq:app-full-size}--\eqref{eq:app-grid-bound} yields
\[
|S'|\le C_T C_n\,dL^2,
\qquad
|B|\le C_T C_n(C_q+C_T)\,dL^3.
\]
Since $L^2\le L^3/\ln2$, their sum is at most
\[
C_T C_n\left(\frac1{\ln2}+C_q+C_T\right)dL^3.
\]
This proves the observation bound $O(d\ln^2(2/\alpha))$, the bit bound $O(d\ln^3(2/\alpha))$, and the total bound in the proposition. Every displayed constant is universal. The inequalities hold over the entire interval $0<\alpha\le1$, and none uses $m$.

If separately prescribed integer budgets are desired, take
\[
k'=Tn,\qquad k''=Tn(b_q+T).
\]
These depend only on the fixed class and accuracy. Every compressor output satisfies $|S'|\le k'$ and $|B|\le k''$, while $k'+k''$ has the preceding total bound. There is no need to pad a shorter output to either maximum budget.

Finally, when $d=0$, the compressor stores the empty observation sequence and the empty binary string. Its total size is exactly zero, as asserted.
\end{proof}

\subsection{Proof of Theorem~\ref{thm:main}}\label{app:closure}
The theorem is the final conclusion of Section~\ref{sec:close}. Lemma~\ref{lem:zero} handles the empty-code branch. In positive dimension, Propositions~\ref{prop:target}, \ref{prop:weak}, and~\ref{prop:boost} define the internal target and accurate ensemble. Proposition~\ref{prop:protocol} identifies the decoded predictor with that ensemble's median. The loss interface of Proposition~\ref{prop:target}, based on Lemma~\ref{lem:loss}, and the budget in Proposition~\ref{prop:size} then give the two conclusions. We specify the maps before combining these guarantees.

\begin{proof}
Fix a nonempty class $\F$ and $0<\alpha\le1$ satisfying the assumptions of the theorem. The external constants in \eqref{eq:constants} are universal. Consequently
\[
c'=\frac{c_2}{128C_2}>0
\]
is universal, with $c'\le1/128$. By hypothesis, $d=\fat(\F,c'\alpha)$ is finite. Fix the orders and numerical parameters once for this class and accuracy. These shared choices are independent of the sample to be compressed.

\stagelead{Definition and existence of the two maps.}
If $d=0$, define $\kappa_\alpha(S)=(S_{\emptyset},B_{\emptyset})$ for every sample $S$, where both components are empty sequences. Define $\rho_\alpha(S',B)=f_0=\Sel(\F)$ on every code. Lemma~\ref{lem:zero} gives
\[
L_2(f_0,S)\le\opt(S)+\alpha/32
\]
for every nonempty sample. This branch has deterministic maps, output in $\F$, and total size zero. It satisfies all conclusions of the theorem.

Suppose $d\ge1$. Define $\kappa_\alpha(S)$ by the construction in Section~\ref{sec:construction}: select the internal target, perform all $T$ fixed-step rounds, and store the subsequence and the string in Section~\ref{sec:protocol}. Define $\rho_\alpha$ on every code by Appendix~\ref{app:decode}.

For completeness, these instructions are defined for every finite sample $S$ of positive length. Proposition~\ref{prop:target} gives a nonempty set of admissible internal targets. Starting from $P_1(i)=1/m$, Proposition~\ref{prop:weak} supplies a good tuple and selected weak function for every reached distribution. The update guarantee in Lemma~\ref{lem:fixed-update}, or equivalently the well-definedness conclusion of Proposition~\ref{prop:boost}, ensures that the next distribution exists. Induction defines every one of the prescribed $T$ rounds. No additional stopping rule is needed. Proposition~\ref{prop:protocol} proves that the resulting code is legal and that these choices define a deterministic compression map. The decoder is a total deterministic map by its definition.

\stagelead{Identify the decoded predictor before comparing losses.}
Let $S$ now be any such sample, and write
\[
(S',B)=\kappa_\alpha(S),\qquad
\hat f=\rho_\alpha(S',B).
\]
Let $g$ be the upper-median predictor specified by the compressor. Proposition~\ref{prop:protocol} gives the equality
\[
\hat f(x)=g(x)\qquad(x\in\X).
\]
Both functions take values in $[0,1]$. Proposition~\ref{prop:boost} supplies a set $G\subseteq[m]$ satisfying
\[
|[m]\setminus G|\le\eps m,
\qquad |g(x_i)-f^*(x_i)|\le\eta/2\quad(i\in G).
\]
The same inequalities therefore hold with $\hat f$ in place of $g$. This step applies the approximation guarantee to the actual output of the reconstruction map.

\stagelead{Transfer the approximation to the original squared loss.}
Proposition~\ref{prop:target} now applies to $\hat f$. We spell out the comparison to retain the precise excess-loss bound. Lemma~\ref{lem:loss} gives a loss increase of at most $\eta$ at each index in $G$. At every remaining index the increase is at most $1$, since both squared losses lie in $[0,1]$. Averaging these inequalities and using the choice of the internal target yields
\begin{align*}
L_2(\hat f,S)
&\le L_2(f^*,S)+\eta\frac{|G|}{m}
                    +\frac{|[m]\setminus G|}{m}\\
&\le\opt(S)+\frac\alpha8+\eta+\eps.
\end{align*}
Here $C_2\ge1$ gives $\eta\le\alpha/12$. Also $\eps=\alpha^2/64\le\alpha/64$ because $\alpha\le1$. Consequently
\begin{equation}\label{eq:app-final-loss}
L_2(\hat f,S)
\le\opt(S)+\left(\frac18+\frac1{12}+\frac1{64}\right)\alpha
=\opt(S)+\frac{43\alpha}{192}
<\opt(S)+\alpha.
\end{equation}
The observations used in this comparison retain the original labels $y_i$. Those labels were not assumed to be values of a single function in $\F$.

\stagelead{Apply the complete description bound.}
Proposition~\ref{prop:size} gives a universal constant $C$ for which
\[
|S'|+|B|\le C d\ln^3(2/\alpha).
\]
It counts both the retained original observations and all auxiliary bits. Proposition~\ref{prop:protocol} supplies the original-subsequence and exact-reconstruction properties, and the maps above are deterministic. The bounds hold for the arbitrary sample $S$, with fixed constants and budgets that do not depend on $m$. Together with the zero-dimensional branch, they prove Theorem~\ref{thm:main}.
\end{proof}

\section{Proofs of the Accuracy Lower Bound}\label{app:accuracy-lower}
This appendix proves the accuracy obstruction in
Section~\ref{sec:accuracy-lower} (Theorem~\ref{thm:accuracy-lower})
and its consequence for the proceedings formulation
(Corollary~\ref{cor:no-log}). The lower-bound argument is independent
of the upper-bound construction and uses samples of one common length.

\subsection{Proof of Theorem~\ref{thm:accuracy-lower}}\label{app:accuracy-lower-proof}
This theorem supplies the unavoidable precision cost in Section~\ref{sec:accuracy-lower}. Its proof uses the compression model in Definition~\ref{def:compression} and the fat-shattering definition. We first relate excess loss to the distance from a sample mean. We then choose separated means that require distinct codes and compare their number with the number of codes of bounded total size. No upper-bound or boosting lemma is used.

\begin{proof}
Use the singleton domain $\X_0=\{x_0\}$ and the class $\mathcal C$ in the theorem. The threshold $1/2$, together with $f_0$ and $f_1$, witnesses that $\{x_0\}$ is $t$-fat-shattered whenever $0<t\le1/2$. No set of two distinct inputs exists in $\X_0$. Consequently,
\[
\fat(\mathcal C,t)=1\qquad(0<t\le1/2).
\]

Fix $0<\alpha<1/4$, and fix any deterministic scheme satisfying the assumptions with integer size bound $k\ge0$. The scheme may have been chosen separately for this value of $\alpha$. For a sample $S$, write
\[
\opt_{\mathcal C}(S)=\inf_{f\in\mathcal C}L_2(f,S).
\]

\stagelead{A finite collection of separated empirical means.}
Set
\[
m=\left\lceil\frac1{2\sqrt\alpha}\right\rceil-1.
\]
Since $\alpha<1/4$, this is a positive integer. The defining property of the ceiling gives
\[
m<\frac1{2\sqrt\alpha}\le m+1,
\qquad\text{and hence}\qquad \frac1m>2\sqrt\alpha.
\]
For each $j\in\{0,\ldots,m\}$, let $S_j$ be the length-$m$ sequence consisting of $m-j$ copies of $(x_0,0)$ followed by $j$ copies of $(x_0,1)$. A block with no copies is empty. These are legitimate agnostic samples: repeated inputs and conflicting labels are allowed. Their label means are
\[
p_j=\frac jm,\qquad j=0,\ldots,m.
\]

For any prediction function $h$ on $\X_0$, put $z=h(x_0)$. Direct expansion yields
\[
L_2(h,S_j)
=(1-p_j)z^2+p_j(z-1)^2
=p_j(1-p_j)+(z-p_j)^2.
\]
The function $f_{p_j}$ belongs to $\mathcal C$, so
\[
\opt_{\mathcal C}(S_j)=p_j(1-p_j),
\qquad
L_2(h,S_j)-\opt_{\mathcal C}(S_j)=(z-p_j)^2.
\]
Thus, if $h_j$ is the reconstruction from the code of $S_j$, the approximation guarantee forces
\[
|h_j(x_0)-p_j|\le\sqrt\alpha.
\]

The $m+1$ samples must produce distinct codes. Indeed, if $S_j$ and $S_\ell$ with $j\ne\ell$ produced the same code, the deterministic decoder would give the same prediction $z$ for both. The triangle inequality would then imply
\[
\frac1m\le |p_j-p_\ell|
\le |p_j-z|+|z-p_\ell|
\le2\sqrt\alpha,
\]
contrary to the choice of $m$.

\stagelead{Counting all available codes.}
Every retained observation from any $S_j$ is either $(x_0,0)$ or $(x_0,1)$. For fixed nonnegative integers $r,b$ with $r+b\le k$, a code with $r$ retained observations and $b$ auxiliary bits therefore has at most $2^r2^b$ possible values. The factor $2^r$ counts every length-$r$ sequence over these two observation values. It therefore allows arbitrary reordering and repetition, with each occurrence counted toward $r$. Requiring an original subsequence or a retained submultiset can only reduce the count.

The code contains observation values, not the original positions from which those values were selected. Selecting an identical observation at a different position does not produce a different retained sequence. Any additional position information would have to be included among the charged bits. We also allow the decoder to know both field lengths without charging extra delimiters; the following sum includes every possible pair of lengths.

The total number of possible codes is at most
\begin{align*}
M_k
&=\sum_{\substack{r,b\ge0\\r+b\le k}}2^{r+b}
=\sum_{s=0}^k(s+1)2^s\\
&=k2^{k+1}+1
\le(k+1)2^{k+1}
\le2^{2k+1}.
\end{align*}
For completeness, the finite-sum identity holds at $k=0$. For $k\ge1$, adding $(k+1)2^k$ to the preceding value $(k-1)2^k+1$ proves it by induction. The last inequality uses $k+1\le2^k$, which holds for every integer $k\ge0$, including $k=0$.

There are $m+1$ distinct produced codes, so
\[
\frac1{2\sqrt\alpha}\le m+1\le M_k\le2^{2k+1}.
\]
Taking base-two logarithms and rearranging gives
\[
\frac12\log_2\frac1\alpha-1\le2k+1,
\qquad
k\ge\frac14\log_2\frac1\alpha-1.
\]
This proves the theorem. The entire comparison used one fixed $\alpha$ and one fixed $m$. In particular, allowing the decoder to know both does not affect the argument.
\end{proof}

\subsection{Proof of Corollary~\ref{cor:no-log}}\label{app:no-log-proof}
This corollary gives the negative answer to the literal proceedings target discussed in Section~\ref{sec:accuracy-lower}. Its sole quantitative input is Theorem~\ref{thm:accuracy-lower}. For the fixed constant class, we choose a scale where its fat dimension is one and the required precision cost exceeds the proposed universal size bound.

\begin{proof}
Suppose, to the contrary, that universal constants $C,c>0$ with the stated property exist. Put $K=\lfloor C\rfloor$, which is a nonnegative integer. Use the fixed class $\mathcal C$ from Theorem~\ref{thm:accuracy-lower}, and choose
\[
0<\alpha<
\min\left\{\frac14,\frac1{2c},2^{-4(K+1)}\right\}.
\]
All three entries are positive, so such a choice exists. Since $0<c\alpha<1/2$, the theorem gives $\fat(\mathcal C,c\alpha)=1$.

The assumed bound would therefore provide a deterministic $\alpha$-approximate scheme for $\mathcal C$ of total size at most $C$ on every finite sample. Each code size is an integer, so its size is in fact at most $K$. Applying Theorem~\ref{thm:accuracy-lower} gives
\[
K\ge\frac14\log_2\frac1\alpha-1>K.
\]
The strict inequality follows from $\alpha<2^{-4(K+1)}$ and is a contradiction. The constants $C,c$ were arbitrary; consequently no universal pair can satisfy the claimed bound.
\end{proof}

\stagelead{A matching order of accuracy dependence for this class.}
The constant class itself admits a scheme of size $O(\log(1/\alpha))$ as $\alpha$ tends to zero. To see this for all samples, including those with arbitrary labels in $[0,1]$, fix $0<\alpha<1/4$ and put
\[
Q=\left\lceil\frac1{2\sqrt\alpha}\right\rceil,
\qquad b=\lceil\log_2(Q+1)\rceil.
\]
Given a nonempty sample with label mean $p$, choose the nearest grid value $j/Q$ with $j\in\{0,\ldots,Q\}$, breaking ties toward the smaller integer. Retain the empty observation sequence and encode $j$ in exactly $b$ bits. The decoder outputs $f_{j/Q}$; all other codes may decode to $f_0$. Both maps know $Q$ from $\alpha$. For arbitrary labels, expansion about their mean gives
\[
L_2(f_{j/Q},S)-\opt_{\mathcal C}(S)
=(j/Q-p)^2
\le\frac1{4Q^2}\le\alpha.
\]
The total size is $b=O(\log(1/\alpha))$. Together with Theorem~\ref{thm:accuracy-lower}, this shows that the optimal total size for this class is $\Theta(\log(1/\alpha))$.

\section{Power Losses and the Scope of the Comparison}\label{app:extensions}
Corollary~\ref{cor:lp} supports the extension in Section~\ref{sec:lp}.

\subsection{Proof of Corollary~\ref{cor:lp}}\label{app:lp}
It reuses the construction from Sections~\ref{sec:target}--\ref{sec:protocol} at a smaller scale. The diameter conclusion of Lemma~\ref{lem:zero} handles the zero-dimensional case. In positive dimension, the infimum argument of Proposition~\ref{prop:target} supplies the internal target. The arguments of Lemmas~\ref{lem:restriction}, \ref{lem:tolerant}, \ref{lem:tuple}, and~\ref{lem:weak-selection} supply the weak functions at the adjusted scale. Lemmas~\ref{lem:fixed-update}, \ref{lem:median}, and~\ref{lem:potential} control their equal-weight median. The proof of Proposition~\ref{prop:protocol} identifies that median with the decoded function, and the counting argument of Proposition~\ref{prop:size} bounds its description. Below we verify the changed parameters and prove the power-loss comparison separately. In particular, we do not apply the squared-loss conclusion of Proposition~\ref{prop:target} to a different loss.

\begin{proof}
Fix $1\le p<\infty$ and $0<\alpha\le1$, and suppose $d_p=\fat(\F,c'\alpha/p)<\infty$. For each nonempty sample $S$, write $\opt_p(S)=\inf_{f\in\F}L_p(f,S)$. For $a,b,y\in[0,1]$, set $u=|a-y|$ and $v=|b-y|$. Both lie in $[0,1]$. The mean value theorem for $z\mapsto z^p$ gives $|u^p-v^p|\le p|u-v|$ for $p>1$, and the same inequality is immediate for $p=1$. The reverse triangle inequality yields $|u-v|\le|a-b|$. Thus
\begin{equation}\label{eq:power-lipschitz}
\bigl||a-y|^p-|b-y|^p\bigr|\le\min\{p|a-b|,1\}.
\end{equation}
The bound by $1$ uses the range of the two losses.

\stagelead{Zero dimension.}
If $d_p=0$, fix $f_0=\Sel(\F)$. The diameter part of Lemma~\ref{lem:zero}, with $t=c'\alpha/p$, gives $|f_0(x)-f(x)|\le2c'\alpha/p$ for all $x$ and every $f\in\F$. Applying \eqref{eq:power-lipschitz} at each sample position gives
\[
L_p(f_0,S)\le L_p(f,S)+2c'\alpha\quad(f\in\F).
\]
Taking the infimum on the right shows $L_p(f_0,S)\le\opt_p(S)+2c'\alpha\le\opt_p(S)+\alpha/64$. The same fixed output and empty code therefore work. This case is completed before any estimate that uses $d_p\ge1$.

\stagelead{Positive dimension and selection of the target.}
Assume $d_p\ge1$ and fix a nonempty sample $S$. Define
\[
\F_{p,S}^*=\{f\in\F:L_p(f,S)\le\opt_p(S)+\alpha/8\}.
\]
The set of attainable losses is nonempty and lies in $[0,1]$. If every attainable loss were at least $\opt_p(S)+\alpha/8$, this number would be a lower bound strictly larger than the infimum. Thus some attainable loss is strictly smaller, and $\F_{p,S}^*$ is nonempty. The fixed rule defines $f^*=\Sel(\F_{p,S}^*)$. This is the infimum argument from Proposition~\ref{prop:target}, applied to $L_p$. It gives
\[
L_p(f^*,S)\le\opt_p(S)+\alpha/8.
\]
Keep this target fixed throughout all rounds. Its values, rather than the original labels, are used in the weak-learning constraints.

\stagelead{The weak-learning scale.}
Set
\begin{align*}
\eta&=\frac{\alpha}{12pC_2},\qquad
\eps=\frac{\alpha^2}{64},\qquad
J=\lceil16/\eta\rceil,\qquad b_q=\lceil\log_2(J+1)\rceil,\\
n&=\left\lceil\frac{8c_1}{3}\left(d_p\ln\frac{32c_3}{3\eta}+\ln2\right)\right\rceil.
\end{align*}
These parameters are fixed by $\F,p,\alpha$ and the universal constants. They do not depend on $S$. In particular, $0<\eta\le1/12$.

For the restriction class $\GS$ on $[m]$, Lemma~\ref{lem:restriction} gives
\[
\fat(\GS,c'\alpha/p)\le d_p.
\]
Apply Lemma~\ref{lem:tolerant} with target $g_{f^*}$ and parameters
\[
\eta_6=\eta/2,\qquad \beta=3/8,\qquad a=1/2,\qquad\delta=1/2.
\]
The dimension scale in that lemma is exactly
\[
c_2\eta_6\beta(1-a)
=c_2\frac\eta2\frac38\frac12
=\frac{3c_2\eta}{32}
=\frac{c_2\alpha}{128pC_2}
=\frac{c'\alpha}{p}.
\]
Its logarithmic term is $\ln(32c_3/(3\eta))$. This is positive, so replacing the restricted dimension by its upper bound $d_p$ gives the displayed sufficient tuple length $n$. The argument of Lemma~\ref{lem:tuple} consequently supplies a good tuple for every distribution $P$ on $[m]$. Choose its lexicographically least element.

Quantize the target on the tuple's support $V$ using the grid with denominator $J$. The nearest-grid error is at most $1/(2J)\le\eta/32$. Hence $f^*$ belongs to the constraint set $H(V,q)$, and its selected function $h$ satisfies
\[
|h(x)-f^*(x)|\le\eta/16+\eta/32<\eta/4
\qquad(x\in V).
\]
The good-tuple implication gives
\[
P\{i:|h(x_i)-f^*(x_i)|>\eta/2\}\le3/8.
\]
These are precisely the steps of Lemma~\ref{lem:weak-selection}. They also show that $V$ and its grid indices determine the entire function $h$. No property of the original-label loss is used in these steps.

\stagelead{Fixed-step boosting and exact reconstruction.}
Use the constants and round budget
\[
a_0=\frac12\ln\frac53,\qquad
\lambda_0=\frac{\sqrt{15}}4,\qquad
T=\left\lceil\frac{\ln(1/\eps)}{-\ln\lambda_0}\right\rceil.
\]
They are the same as in Section~\ref{sec:boost}, since $\eps$ has not changed. Start with the uniform distribution on $[m]$. At each round choose the weak function just constructed and perform the update~\eqref{eq:update} using the present $\eta$ and fixed target $f^*$. Lemma~\ref{lem:fixed-update} applies because the bad mass is at most $3/8$. It gives a positive normalizer bounded by $\lambda_0$ and preserves positive mass at every sample position. Induction therefore defines all $T$ rounds, including any round whose bad mass is zero.

Form the upper median $g$ in~\eqref{eq:equal-median}, with every round counted once. Lemma~\ref{lem:median} gives its existence and the implication from an inaccurate median to at least half of the weak predictions being inaccurate. The potential argument of Lemma~\ref{lem:potential} then yields
\[
\frac1m\bigl|\{i:|g(x_i)-f^*(x_i)|>\eta/2\}\bigr|
\le\lambda_0^T\le\eps.
\]
The argument uses approximation of the fixed target and the weak-error bound, not its empirical $L_2$ loss. Thus it remains valid for the target selected using $L_p$.

Let $V=\bigcup_{t=1}^T V_t$ and $U=|V|\le Tn$. Apply the protocol from Section~\ref{sec:protocol} with the present $J,T,b_q$. It retains one original labeled observation per input in $V$, in original sample order, and records the grid index and $T$ membership bits for each input. The same input has the same target grid value in every round. As in the proof of Proposition~\ref{prop:protocol}, the decoder recovers each $V_t$ and its constraint set exactly. The common selector consequently gives $h_t^{\mathrm{rec}}=h_t$ on all of $\X$. Applying the same upper-median rule proves $\hat f=g$ on $\X$. For this extension, the shared data are $\F,p,\alpha$, and the decoder branches on $d_p$ rather than $d$. In the positive-dimensional branch it uses the new fixed parameter values and the same treatment of malformed codes. It is therefore a total deterministic map for these data. This verifies exact reconstruction for the modified construction.

\stagelead{Loss guarantee.}
For the actual output $\hat f$, there is $G\subseteq[m]$ with $|[m]\setminus G|\le\eps m$ and $|\hat f(x_i)-f^*(x_i)|\le\eta/2$ on $G$. On $G$, inequality~\eqref{eq:power-lipschitz} bounds the increase in loss by $p\eta/2$. On its complement, the increase is at most $1$. Averaging gives
\[
L_p(\hat f,S)-L_p(f^*,S)
\le\frac{p\eta}{2}\frac{|G|}{m}
   +\frac{|[m]\setminus G|}{m}
\le p\eta/2+\eps.
\]
Combining this with the choice of the target yields
\begin{align*}
L_p(\hat f,S)
&\le L_p(f^*,S)+p\eta/2+\eps\\
&\le\opt_p(S)+\frac\alpha8+\frac{\alpha}{24C_2}+\frac{\alpha^2}{64}\\
&\le\opt_p(S)+\frac{35\alpha}{192}<\opt_p(S)+\alpha.
\end{align*}

\stagelead{Size uniform in $p$.}
Let $A=\ln(2/\alpha)$ and $B=\ln(2p/\alpha)$. Then $B\ge A\ge\ln2$. We verify each parameter bound in the field count of Appendix~\ref{app:size} with constants independent of $p$.

First,
\[
\ln\frac{32c_3}{3\eta}
=\ln\frac{128c_3pC_2}{\alpha}
=B+\ln(64c_3C_2).
\]
The last additive term is a fixed universal constant. The ceiling in $n$ adds at most $1$. Since $d_p\ge1$ and $B\ge\ln2$, the terms $\ln2$ and $1$ can also be absorbed into a universal multiple of $d_pB$. Thus $n\le C_n d_p B$ for a universal $C_n$.

Next, $\ln(1/\eps)=2\ln(1/\alpha)+\ln64\le6A$, so
\[
T\le\frac{6A}{-\ln\lambda_0}+1\le C_T A
\]
for a universal $C_T$. For the grid width,
\[
J+1\le\frac{16}{\eta}+2
=\frac{192pC_2}{\alpha}+2
\le(192C_2+2)\frac p\alpha.
\]
It follows that $b_q\le1+\log_2((192C_2+2)p/\alpha)\le C_qB$ for a universal $C_q$. These constants depend only on the external universal constants and the fixed boosting coefficient.

There are exactly $U$ retained observations and $U(b_q+T)$ auxiliary bits. Writing $B_{\mathrm{code}}$ for the auxiliary string, we obtain
\begin{align*}
|S'|+|B_{\mathrm{code}}|
&=U(1+b_q+T)\le Tn(1+b_q+T)\\
&\le C_TC_n d_pAB(1+C_qB+C_TA)\\
&\le C d_pAB^2.
\end{align*}
The last inequality uses $A\le B$ and $1\le B/\ln2$. The constant $C$ is universal, including with respect to $p$. The zero-dimensional case was handled separately with an empty code. This proves the asserted size bound in all cases.
\end{proof}

The finiteness assumption in Corollary~\ref{cor:lp} concerns the smaller scale $c'\alpha/p$. For $p>1$ it does not follow from finiteness at scale $c'\alpha$. To see this, take an infinite domain and choose $a$ with $2c'\alpha/p\le a<2c'\alpha$. The class $\{0,a\}^{\X}$ has no shattered singleton at scale $c'\alpha$, but shatters every finite set at scale $c'\alpha/p$, with thresholds all equal to $a/2$. This also explains why the extension states its own dimension hypothesis.

\subsection{An Exponential Primal--Dual Dimension Gap}\label{app:comparison}
This example supports the comparison in Section~\ref{sec:related}. It is separate from the proof of the compression theorem. The general study of primal and dual dimensions includes the bounds of \citet{assouad1983densite,kleer2021primal}; the following binary construction makes an exponential gap explicit.

\begin{lemma}[A binary primal--dual gap]\label{lem:gap}
For every integer $r\ge1$, there is a finite binary-valued class $\F$ with
$\fat(\F,t)=r$ and $\fatd(\F,t)=2^r$ for every $0<t\le1/2$.
\end{lemma}
\begin{proof}
Let $k=2^r$, $\X=2^{[k]}$, and $\F=\{f_i:i\in[k]\}$, where $f_i(x)=\ind{i\in x}$. These $k$ functions are distinct. If $\F$ $t$-shatters $q$ points for $t>0$, each function realizes at most one sign pattern on the fixed points and thresholds, since the positive and negative requirements at any coordinate are disjoint. There are $2^q$ patterns. Hence $2^q\le k$ and $q\le r$.

For the reverse inequality, index the $k$ functions by binary vectors $b\in\{0,1\}^r$. For $j\in[r]$, choose the input $x_j$ consisting of all indices whose $j$-th binary coordinate is $1$. The evaluation vector of $f_b$ on these inputs is exactly $b$. Thus every binary vector occurs. Taking all thresholds equal to $1/2$ shows that the $r$ inputs are shattered for every $0<t\le1/2$. The primal dimension equals $r$.

For the dual class, its domain is the set of $k$ functions. Given any sign pattern $\sigma\in\{-1,+1\}^k$, choose $x=\{i:\sigma_i=+1\}$. The evaluation function $x^*$ has value $1$ at exactly the positively signed $f_i$. With thresholds $1/2$, these evaluation functions shatter the entire dual domain at every $0<t\le1/2$. The dual dimension is at least $k$, and it cannot exceed its domain size $k$. It is therefore exactly $2^r$.
\end{proof}

\end{document}